\documentclass[preprint,10pt]{elsarticle}

\makeatletter
\def\ps@pprintTitle{%
 \let\@oddhead\@empty
 \let\@evenhead\@empty
 \def\@oddfoot{\centerline{\thepage}}%
 \let\@evenfoot\@oddfoot}
\makeatother

\journal{Physica D}

\usepackage{setspace}
\usepackage{amssymb} 
\usepackage{amsmath}
\usepackage{mathrsfs}
\usepackage{stmaryrd}
\usepackage{amsthm}
\usepackage{mathtools}
\usepackage{amsfonts}
\usepackage{graphicx}
\usepackage{comment}

\usepackage[colorinlistoftodos]{todonotes}
\usepackage{tikz-cd}   
\usepackage[utf8]{inputenc} % allow utf-8 input
\usepackage[T1]{fontenc}    % use 8-bit T1 fonts
\usepackage{hyperref}       % hyperlinks
\usepackage{url}            % simple URL typesetting
\usepackage{booktabs}       % professional-quality tables
\usepackage{amsfonts}       % blackboard math symbols
\usepackage{nicefrac}       % compact symbols for 1/2, etc.
\usepackage{microtype}      % microtypography
\usepackage[utf8]{inputenc} % allow utf-8 input
\usepackage[T1]{fontenc}    % use 8-bit T1 fonts
\usepackage{hyperref}       % hyperlinks
\usepackage{url}            % simple URL typesettinghttps://www.overleaf.com/project/5ba8ce2f6384a04682aa76d1
\usepackage{booktabs}       % professional-quality tables
\usepackage{amsfonts}       % blackboard math symbols
\usepackage{nicefrac}       % compact symbols for 1/2, etc.
\usepackage{microtype}      % microtypography
\usepackage[T1]{fontenc}    % use 8-bit T1 fonts
\usepackage{hyperref}       % hyperlinks
\usepackage{url}            % simple URL typesetting
\usepackage{booktabs}       % professional-quality tables
\usepackage{amsfonts}       % blackboard math symbols
\usepackage{nicefrac}       % compact symbols for 1/2, etc.
\usepackage{microtype}      % microtypography
\usepackage{lingmacros}
\usepackage{tree-dvips} 
\usepackage{amsmath}  
\usepackage{nccmath}
\usepackage{mathtools}
\usepackage{bbm}
\usepackage{amssymb}
\usepackage{graphicx}  
\DeclareMathOperator*{\argmax}{argmax}
\usepackage[utf8]{inputenc} % allow utf-8 input
\usepackage[T1]{fontenc}    % use 8-bit T1 fonts
\usepackage{hyperref}       % hyperlinks
\usepackage{url}            % simple URL typesetting
\usepackage{booktabs}       % professional-quality tables
\usepackage{amsfonts}       % blackboard math symbols
\usepackage{nicefrac}       % compact symbols for 1/2, etc.
\usepackage{microtype}      % microtypography
\usepackage{amsmath, amssymb, mathtools, bm, etoolbox}
\usepackage[noend]{algpseudocode}
\usepackage{algorithm}
\usepackage {mathrsfs}
\usepackage{caption}
\usepackage{graphicx}
\usepackage{subcaption}
\usepackage{color}
\usepackage{enumitem}

\newtheorem{theorem}{Theorem}[section]
\newtheorem{thm-defn}[theorem]{Theorem/Definition}

\newtheorem{prop}[theorem]{Proposition}

\theoremstyle{definition}
\newtheorem{theorem1}{Theorem}[section]
\newtheorem{definition}[theorem1]{Definition}

\theoremstyle{remark}
\newtheorem{remark}[theorem]{Remark}

\newcommand{\rvline}{\hspace*{-\arraycolsep}\vline\hspace*{-\arraycolsep}}

\newcommand*{\vertbar}{\rule[-1ex]{0.5pt}{2.5ex}}
\newcommand*{\horzbar}{\rule[.5ex]{2.5ex}{0.5pt}}
\newcommand{\ignore}[1]{}{}
\usepackage{graphicx}
\usepackage{amssymb}
\usepackage{lineno}

\begin{document}

\begin{frontmatter}


\title{Novel semi-metrics for multivariate change point analysis and anomaly detection}



\author[label1]{Nick James} %\ead{nicholas.james@sydney.edu.au}
\author[label2]{Max Menzies} \author[label1]{Lamiae Azizi} \author[label1]{Jennifer Chan}

\address[label1]{School of Mathematics and Statistics, University of Sydney, NSW, Australia}
\address[label2]{Yau Mathematical Sciences Center, Tsinghua University, Beijing, China}

\begin{abstract}

This paper proposes a new method for determining similarity and anomalies between time series, most practically effective in large collections of (likely related) time series, by measuring distances between structural breaks within such a collection. We introduce a class of \emph{semi-metric} distance measures, which we term \emph{MJ distances}.
These semi-metrics provide an advantage over existing options such as the Hausdorff and Wasserstein metrics. We prove they have desirable properties, including better sensitivity to outliers, while experiments on simulated data demonstrate that they uncover similarity within collections of time series more effectively. Semi-metrics carry a potential disadvantage: without the triangle inequality, they may not satisfy a ``transitivity property of closeness.'' We analyse this failure with proof and introduce an computational method to investigate, in which we demonstrate that our semi-metrics violate transitivity infrequently and mildly. Finally, we apply our methods to cryptocurrency and measles data, introducing a judicious application of eigenvalue analysis.

\end{abstract}

\begin{keyword}
semi-metrics \sep change-point detection \sep multivariate analysis \sep time series \sep anomaly detection % %{\color{red} LA comment: we can remove metric learning and add one key word about time series}


\end{keyword}

\end{frontmatter}

\linenumbers

\begin{nolinenumbers}
\section{Introduction}
\label{S:1}

Similarity and anomaly detection are widely researched problems in statistics and the natural sciences. Change point detection is an important task in time series analysis, and more broadly, within anomaly detection. Developed by Hawkins et al. \cite{Hawkins1977,Hawkins2005}, this task requires one to estimate the location of changes in statistical properties, points in time at which the estimated probability density functions change. In the more statistical literature, focussed on time series data, statisticians such as Ross \cite{Ross2011T,Ross2012N} have developed change point models driven by hypothesis tests, where $p$-values govern statistical decision making. The change point models applied in this paper follow Ross \cite{RossCPM}. 

Various change point algorithms test for shifts in different underlying distributional properties and generally make strong assumptions regarding the statistical properties of random variables. In this paper, we make use of the \emph{Mann-Whitney} test, described in \cite{Pettitt1979}, which detects changes in the mean, and the Kolmogorov-Smirnov test, described in \cite{Ross2012}, which detects more general distributional changes. These algorithms make the strong assumption of independence; for applicability of the change point algorithms on data with dependence, see transformations due to Gustafsson \cite{gustafsson2001}.

Metric spaces $(X,d)$ appear throughout mathematics. One particular field of study that has arisen in image detection and other applications is the study of metrics on the power set $PX$ of (certain) subsets of $X$. The most utilised metric in this context is the \emph{Hausdorff metric}, which we introduce and summarise in Section\ref{review of existing section}. This provides a metric between closed and bounded subsets of any ambient metric space $(X,d)$. In addition to the Hausdorff metric, there are several \emph{semi-metrics}, satisfying some but not all of the properties of a metric, which are still useful. These properties will also be summarised in Section\ref{review of existing section}.

Conci and Kubrusly \cite{Conci2017} give an overview of certain such (semi-)metrics on the space of subsets and some applications. This review breaks down the applications of such distances between subsets into three primary areas, computational aspects \cite{Eiter1997, Atallah1983, Atallah1991, Barton2010, Shonkwiler1989, Huttenlocher1990, Huttenlocher1992, Aspert2002}, distances between fuzzy sets \cite{Brass2002, fujita2012, Gardner2014, Rosenfeld1985, chaudhuri1996, Chaudhuri1999, Boxer1997, Fan1998} and distances in image analysis \cite{fujita2012, Gardner2014, Dubuisson1994N, Rote1991, Li2008, Huttenlocher1990, Huttenlocher1992, Rucklidge1995, Rucklidge1996, Rucklidge1997}. The Hausdorff's sensitivity to outliers has been noted by \cite{Baddeley1992}, and has proven itself largely unsuitable for algorithmic problems pertaining to image analysis. In Section\ref{results section}, we present similar findings when using the Hausdorff distance between finite sets of change points.

There has been extensive work in determining similarity between time series. Moeckel and Murray \cite{Moeckel1997} survey the shortcomings of possible distance functions between time series, stating the Hausdorff metric's limitation in ignoring the frequency with which one set visits parts of the comparable set. That is, the metric only focuses on one measurement between two candidate sets, and is sensitive to outliers. Instead, they propose a distance function developed by Kantorovich \cite{Kantorovich2006} that is based on geometric and probabilistic factors; this proves to be robust with respect to outliers, noise and discretisation errors. Moeckel and Murray also demonstrate that the transportation distance proposed by Kantorovich can be used to evaluate mathematical models for chaotic or stochastic systems, and for parameter estimation within a dynamical systems context. %Hutchinson \cite{Hutchinson1981} used a related function in the setting of iterated function systems.

To our knowledge, there has been no work in detecting the similarity between time series' change points. Change points signify changes in the statistical properties of a time series' distribution, so determining which time series are most similar with respect to the number and location of these changes is of interest to analysts, in a wide variety of disciplines, who wish to assess the underlying structure in larger collections of time series. Equally of interest are time series whose change points are dissimilar to the rest of the collection and exhibit anomalous behaviour with respect to their structural breaks. 

% rewrite section... new understanding of cryptocurrency.
The contributions of this paper are as follows. First, we introduce a new family of semi-metrics and provide analysis of their properties, and new analysis of existing semi-metrics. Next, we apply such semi-metrics and metrics to measure distance between time series based on their structural breaks, forming a new distance matrix between time series. We introduce a computational method for analysing the transitivity properties of a semi-metric, and perform it in this setting. Finally, we introduce a simple but pithy method of eigenvalue analysis in order to determine the size of a majority cluster in our setting. In circumstances when one expects, \emph{a priori}, a large majority of time series to behave very similarity, with some anomalies, our eigenvalue analysis quickly approximates the size of a majority cluster and provides an understanding of the total scale of the distance matrix.

The rest of the paper is organised as follows. Section \ref{review of existing section} provides a review of existing (semi)-metrics. In Section\ref{new semi metrics section}, we propose a new family of semi-metrics, analyse their desirable properties, and prove propositions on both the new and existing semi-metrics. Section \ref{Section Distance matrix analysis} describes our computational methodology. Section \ref{results section} conducts simulations in three scenarios, analysing the robustness of the metrics and semi-metrics in the presence of outliers. Section \ref{real results} applies our analysis to the cryptocurrency market and 19th century UK measles data. We use our eigenvalue analysis, as well as hierarchical and spectral clustering. Section \ref{conclusion} concludes the paper. In \ref{Appendix_CPD} and \ref{proof of props section} respectively, we provide a description of the change point algorithm used, and include all remaining proofs. 

\section{Review and analysis of existing (semi)-metrics}
\label{review of existing section}
%\subsection{Metrics and semi-metrics between finite sets}
%\label{metric section}
In this section, we review some properties of a metric space and the existing Hausdorff, modified Hausdorff and Wasserstein (semi-)metrics. Most significantly, we describe exactly how the Wasserstein is used between finite sets, which does not appear clearly in the literature.

%\subsection{Review of current metrics and semi-metrics} \label{Review}
A metric space is a pair $(X,d)$ where $X$ is a set and $d: X \times X \to \mathbb{R}$ satisfies the following axioms for all $x,y,z \in X$:
\begin{enumerate}
    \item $d(x,y) \geq 0$, with equality if and only if $x=y$.
    \item $d(x,y) = d(y,x)$. 
    \item $d(x,z) \leq d(x,y)+d(y,z)$
\end{enumerate}
%\subsubsection{The distance from a point to a set}
%\label{min distance defn}

A \emph{semi-metric} satisfies 1. and 2., but not necessarily 3., which is known as the \emph{triangle inequality}.

Given a subset $S \subset X$ and a point $x \in X$, the distance from a point to a set is defined as the minimal distance from $x$ to $S$, given by: 
\begin{equation}
d(x,S) = \inf_{s \in S} d(x,s). \label{min distance defn}
\end{equation}
$d(x,S) \geq 0$, with equality if and only if $x$ lies in the closure of $S$. Also, $d(-,S); X \to \mathbb{R}$ is continuous.
%\subsubsection{The minimal distance between sets}
%\label{min min distance}
Now let $S, T \subset X$. A common notion of distance between these subsets is defined as the minimal distance between these subsets, given by:
\begin{equation}
    d_{\text{min}}(S,T) = \inf_{s \in S} d(s,T) =  \inf_{s \in S} \inf_{t \in T} d(s,t) = \inf_{s \in S, t \in T} d(s,t). \label{min min distance}
\end{equation}    
Note $d_{\text{min}}(S,T) = 0$ if $S,T$ intersect. In fact, $d_{\text{min}}(S,T)=0$ if and only if their closures intersect. So this is not an effective metric between subsets.

%\subsection{Current (semi-)metrics}
%\subsubsection{Hausdorff Distance}
\begin{definition}[Hausdorff distance]
The Hausdorff distance considers how separated $S$ and $T$ are at most, rather than at least. It is defined by:
\begin{eqnarray*}
    d_{H}(S,T) & = & \text{max } \bigg( \sup_{s \in S} d(s,T), \sup_{t \in T} d(t,S) \bigg), \\
    & = & \sup \{ d(s,T), s \in S; d(t,S), t \in T \}. 
\end{eqnarray*}
\end{definition}
This is the supremum or $L^{\infty}$ norm of all minimal distances from points $s \in S$ to $T$ and points $t\in T$  to $S$. The Hausdorff distance satisfies the triangle inequality, but this supremum is highly sensitive to even a single outlier. We propose using the $L^p$ norms instead. Henceforth, $S$ and $T$ will be finite sets. Eventually, $S,T$ will be sets of structural breaks of time series. We present three modified Hausdorff distances below.

\begin{definition}[Modified Hausdorff distance 1]
%\subsubsection{Modified Hausdorff distance 1}
The first modified Hausdorff distance MH$_1$ is defined by
\begin{align*}
d^{\text{MH}}_1(S,T) = \max \bigg( \frac{1}{|S|} \sum_{s \in S} d(s,T), \frac{1}{|T|} \sum_{t \in T} d(t,S) \bigg).
\end{align*}
\end{definition}
\noindent It is presented in \cite{Deza2013} and \cite{Dubuisson1994N}. As is the case with most modified Hausdorff metrics, the primary application to date has been in computer vision tasks, where semi-metrics and metrics focused on geometric averaging provide a more robust distance measure in comparison to the Hausdorff distance. 

\begin{definition}[Modified Hausdorff distance 2]
%\subsubsection{Modified Hausdorff distance 2}
The second modified Hausdorff distance MH$_2$ is defined by
\begin{align*}
d^{\text{MH}}_2(S,T) = \sum_{s \in S} d(s,T) + \sum_{t \in T} d(t,S).
\end{align*}
\end{definition}
\noindent Eiter \cite{Eiter1997} and Dubuisson \cite{Dubuisson1994N} present this distance measure, which captures the total distance between one set and another. Removing the $\sup$ operator yields what is essentially a measure of total deviation between all points of two sets. %This is in fact the most sensitive to outliers, as we will see in Section\ref{outliers}.

\begin{definition}[Modified Hausdorff distance 3]
%\subsubsection{Modified Hausdorff distance 3}
The third modified Hausdorff distance MH$_3$ is defined by
\begin{align*}
d^{\text{MH}}_3(S,T) = \frac{1}{|S|+|T|} \bigg(\sum_{s \in S} d(s,T) + \sum_{t \in T} d(t,S) \bigg).
\end{align*}
\end{definition}
\noindent Deza \cite{Deza2013} and Dubuisson \cite{Dubuisson1994N} propose this as variant of d$^{1}_{MH}$ with a different averaging component. This measure is referred to as geometric mean error between two images.

\begin{definition}[Wasserstein distance]
The Wasserstein metric, \cite{DelBarrio} is commonly used as a distance between two probability measures. Intuitively, it gives the work (in the sense of physics) required to mould one probability measure into another. Given probability measures $\mu,\nu$ on a metric space $(X,d)$, define
\begin{equation*}
    W_{p} (\mu,\nu) = \inf_{\gamma} \bigg( \int_{X \times X} d^{p} (x,y) d\gamma  \bigg)^{\frac{1}{p}}.
\end{equation*}
This infimum is taken over all joint probability measures $\gamma$ on $X\times X$ with marginal probability measures $\mu$ and $\nu$. Now let $S,T$ be finite sets. Associate to each set a probability measure defined as a weighted sum of Dirac delta measures
\begin{align}
\label{Wasserstein delta}
    \mu_S=\frac{1}{|S|}\sum_{s \in S}\delta_s
\end{align}
The Wasserstein distance is defined as $d_W^p(S,T):=W_p(\mu_S,\mu_T)$. In subsequent experiments, when using the Wasserstein metric, we set $p=1$.
\end{definition}

\section{Proposed (semi)-metrics}
\label{new semi metrics section}

In this section, we introduce a new family of semi-metrics, and analyse their properties and advantages over existing options. First, we motivate and introduce the MJ$_1$ semi-metric, then generalise this to the family of MJ$_p$ semi-metrics, which, when properly extended to infinity, includes the Hausdorff metric.

In \cite{Dubuisson1994N}, Jain and Dubuisson assert that their distance MH$_1$ is the best for image matching. To reach this conclusion, they take two steps. First, (page 567) they compare three favourable operators $f_2,f_3,f_4$, each operating on minimal distances $d(s,T),d(t,S)$ as defined in Equation (\ref{min distance defn}). They briefly argue that $f_2$, equivalent to taking the max in the MH$_1$, is preferable to other operators, citing a ``larger spread." Secondly, (page 568) they argue that a process of averaging distances is superior to taking $K$th ranked distances, such as the median. We differ with and modify these steps of reasoning. For the first, we replace the max in their MH$_1$ with the $L^1$ norm average of all the minimum distances from $S$ to $T$ and $T$ to $S$:

%\begin{definition}[New order 1 MJ distance]
%The order 1 MJ distance is given by
\begin{equation*}
    d^1_{MJ}({S},{T}) = \frac{1}{2} \Bigg(\frac{\sum_{t\in T} d(t,S)}{|T|} + \frac{\sum_{{s} \in {S}} d(s,T)}{|S|} \Bigg).
\end{equation*}
%\end{definition}
Then. we show desirable properties of MJ$_1$ over MH$_1$ and MH$_3$ in the following two propositions.

\begin{prop}[Comparison between MJ$_1$ and MH$_1$]\ \\
\label{deformation prop}
MJ$_1$ and MH$_1$ are equivalent as semi-metrics. However, MJ$_1$ is more precise, in the sense that there exists a class of deformations of $S$, such that $d^1_{MJ}$ will vary continuously with $S$ while $d^{\text{MH}}_1$ will not vary at all. 
\end{prop}
\noindent The proof is given in \ref{A2}.

\begin{prop}[Comparison between MJ$_1$ and MH$_2$, MH$_3$]\ \\
\label{duplication prop}
\noindent The following property holds for MJ$_1$ but not MH$_2$ or MH$_3$: if all elements of $S$ are duplicated,  $d_{MJ}^1(S,T)$ does not change, while $d^{\text{MH}}_i(S,T)$ does change, $i=2,3.$
\end{prop}
\noindent If we duplicate elements of a set, the set itself does not change, so a measure between sets should not change under such a duplication. \par \vspace{2mm}
\noindent The proof is given in \ref{A3}. MH$_2$ in particular greatly enlarges with the duplication or addition of points.
\begin{remark}
Unfortunately, if one single point of $S$ is duplicated, $d^1_{MJ}$ does change. This is the case with all existing modified Hausdorff distances. Yet even this is not disastrous, because it reflects that a greater concentration of certain elements of $S$ represents a different distribution of the data points in $S$.
\end{remark}

Regarding the second step of \cite{Dubuisson1994N}, we agree that an averaging process is much less sensitive to outlier error than the alternative processes. However, we may generalise this process by using other $L^p$ norm averages. And so we present generalised semi-metrics below.
\begin{definition} 
\label{New MJ distance}
We define the MJ$_p$ distance by
\begin{equation*}
    d^p_{MJ}({S},{T}) = \Bigg(\frac{\sum_{t\in T} d(t,S)^p}{2|T|} + \frac{\sum_{{s} \in {S}} d(s,T)^p}{2|S|} \Bigg)^{\frac{1}{p}}.
\end{equation*}
\end{definition}
\noindent This is chosen so that $d^p_{MJ}({S},{T}) \leq d_H(S,T)$ for all $p$ and 
$$\lim_{p \to \infty} d^p_{MJ}({S},{T}) = d_H(S,T).$$
Hence, $d_H$ can now be viewed as the $L^\infty$ norm of these distances. Thus, our family of semi-metrics includes the Hausdorff distance as a limiting case when $p \rightarrow \infty$, placing the existing Hausdorff metric in a new family of semi-metrics.

\begin{remark}
Usually in the context of $L^p$ norms, $p$ must be in the range $p \geq 1$ to preserve the triangle inequality. Since these measures do not preserve the triangle inequality, we can take $p>0$. This means that $p=\frac{1}{2}$, for example, is even less sensitive to outliers than the MH$_1$ and MJ$_1$ distances. \\

As $p$ grows larger, $d^p_{MJ}$ approaches the Hausdorff metric, which satisfies the triangle inequality. As $p$ grows smaller, these distances are less sensitive to outliers. Thus, this continuum of $p$ allows us to compromise between the triangle inequality property of the metric, and the sensitivity to outliers. In Section\ref{choosing p section}, we explore the possibility of optimising $p$ under these considerations.

If we were guaranteed that all distances $d(t,S)$ and $d(s,T)$ were non-zero, we could also consider $p\leq 0$. For $p=0$ the above norm is properly interpreted as a limit which equals the geometric mean
\begin{equation*}
    d^0_{MJ}({S},{T}) = \prod_{t\in T} d(t,S)^{\frac{1}{2|T|}} \prod_{{s} \in {S}} d(s,T)^{\frac{1}{2|S|}}.
\end{equation*}
As $p\to - \infty, d^p_{MJ} \to d_{\text{min}}(S,T)$ from equation (\ref{min min distance}). Even this quantity is contained in our new family. However, for $p\leq 0$ note $d^p_{MJ}(S,T)=0$ if $S,T$ intersect, so $d^p_{MJ}$ is not a semi-metric under axiom 1.
\end{remark}

\begin{prop}
For $p>0$, MJ$_p$ satisfies the properties of propositions \ref{deformation prop} and \ref{duplication prop}, exactly like MJ$_1$.
\end{prop}
\begin{proof}
Identical to the proofs of these propositions in \ref{proof of props section}.
\end{proof}

\begin{prop}
\label{triangle inequality again}
For $p >0$, the MJ$_p$ measures are semi-metrics. However, they fail the triangle inequality up to any constant. That is, there is no constant $k$ such that
\begin{align} \label{TriIn}
    d^p_{MJ}(S,R) \leq k(d^p_{MJ}(S,T)+d^p_{MJ}(T,R))
\end{align}
for any subsets $S,T,R$. This also applies to MH$_i$, for $i=1,2,3.$
\end{prop}
\noindent The proof is given in \ref{A1}. We remark that this proof does not exist in the literature even for the existing modified Hausdorff semi-metrics.

\subsection{Sensitivity to outliers}
\label{outliers}
We examine the sensitivity to outliers of all discussed metrics and semi-metrics. Let $T=\{t_1,...,t_n\}$ and $S$ be fixed. Fixing all but one element, if $t_n \to \infty$ acts as an outlier, we examine the effect on all distances $d(S,T)$.

First, the Hausdorff distance $d_H(S,T)$ increases asymptotically with $t_n,$ for $t_n$ sufficiently large, illustrating its unsuitability for outliers. That is,
\begin{align*}
d_H(S,T) \sim |t_n|, \hspace{3mm} \text{meaning} \hspace{3mm} \lim_{|t_n| \to \infty} \frac{d_H(S,T)}{|t_n|} = 1. 
\end{align*}
MH$_2$ contains the term $d(t_n,S)$ which also increases asymptotically with $t_n.$ That is, $d^{\text{MH}}_2(S,T) \sim |t_n|$, illustrating its sensitivity to outliers. Due to the averaging within MH$_1$, MH$_3$ and MJ$_p$, all of these semi-metrics perform well with outliers, but this gets worse as $p$ increases, and better if it decreases. Specifically,

\begin{align*}
d^{\text{MH}}_1(S,T) \sim \frac{|t_n|}{|T|}, d^{\text{MH}}_3(S,T) \sim \frac{|t_n|}{|T|+|S|}, d^p_{MJ}(S,T) \sim \frac{|t_n|}{(2|T|)^{\frac{1}{p}}}, %\hspace{3mm}  \lim_{t_n \to \infty} \frac{d_H(S,T)}{|t_n|} = 1. 
\end{align*}

Finally, we examine a property of the MJ$_p$ family, but not the Wasserstein distance, indicating the latter's unsuitability to measure distance between data sets.

\begin{prop}
\label{intersection prop}
If $|S\cap T|=r,$ the following inequality holds: 
\begin{align*}
d^p_{MJ}(S,T) \leq \left[1-\frac{r}{2}\Big(\frac{1}{|S|}+ \frac{1}{|T|}\Big)\right]^{\frac{1}{p}}d_H(S,T)    
\end{align*}
No such inequality holds for Wasserstein distance. Even with $|S\cap T|=|S|-1=|T|-1$, it is possible for $d_W(S,T)$ to coincide with $d_H(S,T)$.
\end{prop}
\noindent %This result means that substantial intersection in two data sets $S,T$ will yield reduced values of $d^p_{MJ}$, but not $d_W$. 
The proof is given in \ref{A4}.

\begin{remark}
As a consequence of proposition \ref{intersection prop}, if $S$ and $T$ have a large amount of similarity in their elements, $d^p_{MJ}(S,T)$ will reflect this close similarity between $S$ and $T$, while the Wasserstein distance $d_{W}(S,T)$ may not. This will prove useful in analysing data sets in sections \ref{results section},\ref{real results}. We adopt an example from the proof here. Let $A_1=\{0,999\},B_1=\{1,1000\}$ and $A_2=\{0,1,...,999\},B_2=\{1,2,...,1000\}$. Observing sets $A_1,B_1$, clear candidates for distances between them are $1$ and $2$. Indeed, $d_W(A_1,B_1)=d_H(A_1,B_1)=d^{\text{MH}}_1(A_1,B_1)=d^{\text{MH}}_3(A_1,B_1)=d^p_{MJ}(A_1,B_1)=1$ while $d^{\text{MH}}_2(A_1,B_1)=4$.\\

Using the Wasserstein or Hausdorff distance, the separation between $A_2$ and $B_2$ remains $d_W(A_2,B_2)=d_H(A_2,B_2)=1$. This is an appropriate distance from a translational or geometric point of view. However, it ignores the remarkable similarity in the data of $A_2$ and $B_2$. If these were sets of change points, they would be considered remarkably similar. Appropriately, $d^{\text{MH}}_1(A_2,B_2)=\frac{1}{1000}$, $d^p_{MJ}(A_2,B_2)=\big(\frac{1}{1000}\big)^{\frac{1}{p}}$.
\end{remark}

To summarise, we have proven:
\begin{theorem}
There exists a family of semi-metrics MJ$_p$ which include the Hausdorff distance as a limiting member when $p \to \infty$. Like MH$_i$, $i=1,2,3,$ they fail the triangle inequality up to any constant. However, they have a precision advantage over MH$_1$, a duplication-invariance advantage over MH$_3$, and are much more insensitive to outliers than MH$_2$ and the Hausdorff metric. They also are more suitable than the Wasserstein at reflecting high intersection in the data.
\end{theorem}
\begin{proof}
Combine propositions \ref{deformation prop}, \ref{duplication prop}, \ref{triangle inequality again}, \ref{intersection prop}.
\end{proof}

\section{Computational methodology} \label{Section Distance matrix analysis}
To generate our distance matrix $D$, we compute the distance between all sets of time series change points within our collection of time series. Suppose we have $n$ time series with sets of change points $S_1,\ldots,S_n$. We define the following distance matrices. \par \vspace{2mm} 
Hausdorff distance matrix $(D_H)_{ij}$:
\begin{equation}
    (D_H)_{ij} = d_H(S_i,S_j)= \text{max}\bigg( \sup_{x \in S_i} d(x,S_j), \sup_{y \in S_j} d(y,S_i) \bigg) \forall i,j=1,\ldots,n. 
\end{equation}

MH$_1$ distance matrix $(D_{MH_1})_{ij}$:
\begin{equation}
    (D_{MH_1})_{ij} = d^{\text{MH}}_1(S_i,S_j)=\max \bigg( \frac{1}{|S_i|} \sum_{x \in S_i} d(x,S_j), \frac{1}{|S_j|} \sum_{y \in S_j} d(y,S_i) \bigg)
\end{equation}

MH$_2$ distance matrix $(D_{MH_2})_{ij}$:
\begin{equation}
    (D_{MH_2})_{ij} = d^{\text{MH}}_2(S_i,S_j)= \sum_{x \in S_i} d(x,S_j) + \sum_{y \in S_j} d(y,S_i)
\end{equation}

MH$_3$ distance matrix $(D_{MH_3})_{ij}$:
\begin{equation}
    (D_{MH_3})_{ij} = d^{\text{MH}}_3(S_i,S_j) =\frac{1}{|S_i|+|S_j|} \left(\sum_{x \in S_i} d(x,S_j) + \sum_{y \in S_j} d(y,S_i) \right)
\end{equation}

Wasserstein distance matrix $(D_W)_{ij}$:
\begin{equation}
    (D_W)_{ij} = \inf_{\mu} \bigg( \sum_{x \in S_i, y \in S_j} d^{p} (x,y) \mu (x,y) \bigg)^{\frac{1}{p}}
\end{equation}

MJ$_p$ distance matrix $(D_{MJ_p})_{ij}$: 
\begin{equation}
    (D_{MJ_p})_{ij}=d^p_{MJ}(S_i,S_j) =
    \Bigg(\frac{\sum_{y\in S_j} d(y,S_i)^p}{2|S_j|} + \frac{\sum_{x \in S_i} d(x,S_j)^p}{2|S_i|} \Bigg)^{\frac{1}{p}}
\end{equation}

\subsection{Transitivity analysis} \label{Transitivity analysis}
Like the modified Hausdorff distances, the introduced distance measures MJ$_p$ do not satisfy the triangle inequality in generality. This is significant because it is possible that sets $S,T$ and $T,R$ are each close with respect to these measures, but $S,R$ are not close. Then the property of closeness would not be transitive. However, in practice, the distance measures respect transitivity quite well, at least for $p\geq 1$. We examine two questions:
\begin{enumerate}
    \item how often do the new semi-metric distances fail the triangle inequality; 
    \item how badly do these distances violate the triangle inequality.
\end{enumerate}

To explore these two questions, we empirically generate a three dimensional matrix and test whether the triangle inequality is satisfied for all possible combinations of elements within the matrix. We construct our matrix as follows:

\begin{equation}
    T_{i,j,k}=
    \begin{cases}
      \text{blue }, & \frac{D_{ik}}{D_{ij} + D_{jk}} \leq 1, \\
      \text{yellow },  & 1 < \frac{D_{ik}}{D_{ij} + D_{jk}} \leq 2, \\
      \text{red}, & \text{else.}
    \end{cases}
\end{equation}

\subsection{Eigenvalue analysis}
Analysing the eigenvalues and eigenvectors of a system in physical and applied sciences is of great importance. Matrix diagonalisation and matrices' eigenspectra arise in many applications such as stability analysis and oscillations of vibrating systems. %Let A be a linear map represented by a matrix A. If a vector $\boldsymbol{X} \in \mathbb{R}^n \neq 0 $ exists s.t. 
%\begin{equation}
 %   A \boldsymbol{X} = \lambda \boldsymbol{X}
%\end{equation}
%for a scalar value $\lambda$, then $\lambda$ is the eigenvalue of A with a respective eigenvector X. 
In our context, we analyse the distance matrices $D$, all of which are symmetric real matrices with trace $0$. As such, they can be diagonalised over the real numbers with real eigenvalues. To determine similarity of time series with respect to their change points, we plot the absolute value of eigenvalues for all matrices. Note all eigenvalues are real and sum to zero.
\\

Consider the following real world heuristic: many real time series, such as stock returns, are not necessarily highly correlated on a regular basis. The returns of Microsoft and Ford may have little to do with each other over time. However, it is expected that a significant market event or crash would significantly affect both Microsoft and Ford at essentially the same time, and yield a change point in the stochastic properties of both time series at the same time. Thus, even if the overall properties of time series may be uncorrelated or negatively correlated, change points are likely to cluster. It would be of considerable interest if a third time series, say the returns of a new green energy company, had change points different from the majority. Perhaps this third time series would then be concluded to be less vulnerable to a market crash. In our analysis of time series, especially cryptocurrency in Section\ref{real results}, we expect a large majority %say 70\% 
of time series to follow similar change points, and from these we will be able to examine the exceptional ones for any opportunities. 
     
Mathematically, if say $70$ out of $100$ time series have very similar change points, then the distance matrix $D$ should have the following structure:

\[
\begin{pmatrix}
  \begin{matrix}
   & c_{1} & c_{2} & c_{3} & \hdots & c_{70} \\
  r_{1} & 0 & * & * & \hdots & * \\
  r_{2} & * & 0 & * & \hdots & * \\
  r_{3} & * & * & 0 & \hdots & *\\
   & \vdots & \vdots & \vdots & \ddots \\
  r_{70} & * & * & * & * & 0 \\
  \end{matrix}
  & \rvline &   
  \begin{array}{ccc}
    \horzbar & r_{1} & \horzbar \\
    \horzbar & r_{2} & \horzbar \\
             & \vdots    &          \\
    \horzbar & r_{70} & \horzbar
  \end{array} & \\
\hline
  \begin{array}{ccccc}
    \vertbar & \vertbar & \vertbar &        & \vertbar \\
    r^{T}_{1}    & r^{T}_{2} & r^{T}_{3}   & \ldots & r^{T}_{70}    \\
    \vertbar & \vertbar & \vertbar &       & \vertbar 
  \end{array} & \rvline &
  \begin{matrix}
  0 &  \\
   & 0 \\
   &  & \ddots \\
   & & & 0
  \end{matrix}
\end{pmatrix}
\]
where rows $r_1,\ldots,r_{70}$ are highly similar to one another and elements $*$ are close to zero. This means small deformations in the matrix entries exist to make the first $70$ rows identical. Hence, this matrix is a small deformation from a rank $31$ matrix, with $69$ eigenvalues equal to zero. That is, if $70$ of $100$ time series have very similar change points, then $69$ of the eigenvalues should be close to zero.

Given a threshold $\epsilon$, we can rank the absolute values of the eigenvalues $|\lambda_1|\leq...\leq|\lambda_n|$. If $|\lambda_1|,...,|\lambda_k|<\epsilon$ then we can deduce $k+1$ of the time series are similar with respect to their structural breaks. This may be the most pithy way of expressing the number of time series that are similar in terms of their change points within large collections of time series. If we \emph{a priori} have reason to believe that one large majority cluster will exist, a judicious choice of $\epsilon$ can determine its size immediately. One can approximate its size by inspection from the graphical depictions such as Figures \ref{fig:No_Eigenvalue_Analysis},\ref{fig:Several_Eigenvalue_Analysis},\ref{fig:Worst_Eigenvalue_Analysis}. 

Moreover, eigenvalue analysis provides us a quick measure of the scale of the distance matrix. Since all distance (and affinity) matrices are symmetric, $D$ can be conjugated by an orthogonal matrix to give a diagonal matrix of its eigenvalues. This is known as the spectral theorem, \cite{Axler}. As a consequence, the operator norm \cite{RudinFA} coincides with $|\lambda_n|$. That is, $$\max_{x \in \mathbb{R}^n - \{0\}} \frac{||Dx||}{||x||}= ||D||_{op}= |\lambda_n|$$

\begin{remark}
\label{eval plots look sim}
Even when these plots look quite similar, the scale gives us information about the scale of the distance matrices. 
\end{remark}

\subsection{Spectral clustering affinity matrix}
Spectral clustering applies a graph theoretic interpretation of our problem, and projects our data into a lower dimensional space, the eigenvector domain, where it may be more easily separated by standard algorithms such as $K$-means.
Following \cite{vonLuxburg2007}, we transform our distance matrix $D$ into an affinity matrix $A$ as follows:
\begin{equation}
    A_{ij} = 1 - \frac{D_{ij}}{\max_{kl} D_{kl}}, \hspace{3mm} \forall i,j=1,\dots,n. 
\end{equation}
The graph Laplacian matrix is given by:
\begin{equation}
    L = E - A, 
\end{equation}
where $E$ is the diagonal degree matrix with diagonal entities  $E_{ii} = \sum_{j} A_{ij}$. $A$ and hence $L$ are real symmetric matrices, so can be diagonalised with all real eigenvalues. In particular, $L$ is \emph{positive semi-definite} with eigenvalues $0=\lambda_{1} \leq \lambda_{2} \leq ... \leq \lambda_{n}.$

\begin{comment}
The graph Laplacian matrix possesses four key properties:
\begin{enumerate}
    \item $f^TLf = \frac{1}{2} \sum_{i,j} A_{ij} (f_{i} - f_{j}), \forall f \in \mathbb{R}^n$, where $f$ is any candidate eigenvector.
    \item $L$ is symmetric and positive semi-definite.
    \item The smallest eigenvalue of $L$ is $0$ and the respective eigenvector is $(1,\ldots,1)^T$
    \item Its eigenvalues are non-negative real numbers $0=\lambda_{1} \leq \lambda_{2} \leq ... \leq \lambda_{n}.$
\end{enumerate}
\end{comment}

Spectral clustering proceeds as follows. With $k$ chosen \textit{a priori}, find corresponding eigenvectors $f_{1}, f_{2},...,f_{k}$ and construct the matrix $F \in \mathbb{R}^{n \times k}$ whose columns are $f_i,i=1,\dots,k$. Let $v_j \in \mathbb{R}^k$ be the rows of $F, j=1,\dots,n$. Apply standard $K$-means to cluster these rows into clusters $C_1,...,C_k$. Finally, output clusters $A_l=\{i: v_i \in C_l \}, l=1,...,k$ to assign the original $n$ elements into the corresponding clusters.

\subsection{Dendrogram analysis}
A dendrogram displays the hierarchical relationships between objects in a dataset. Hierarchical clustering falls into two categories:
\begin{enumerate}
    \item Agglomerative clustering - a bottom-up approach where all data points start as individual clusters; or
    \item Divisive clustering - a top-down approach where all data points start in the same cluster and are recursively split.
\end{enumerate}
The dendrogram and hierarchical clustering results are highly dependent on the distance measure used to determine clusters. We display the respective dendrogram of our eight candidate distance matrices and assess which method displays similarity between time series most appropriately for our change point problem. The colours of the dendrogram indicate the closeness of any two sets of time series change points.

%\section{Results} \label{results section}
\section{Simulation study} \label{results section}

We generate three collections of ten time series. The first collection exhibits very few change point outliers, the second exhibits a moderate number of less severe change point outliers and the final collection of time series exhibits multiple extreme change point outliers. The Hausdorff, three modified Hausdorff varieties, Wasserstein, MJ$_{0.5}$, MJ$_1$ and MJ$_2$ distances are compared between the time series.

\subsection{Simulation $1$: no change point outliers}

\begin{figure}[t]
    \centering
    \includegraphics[width=.48\textwidth]{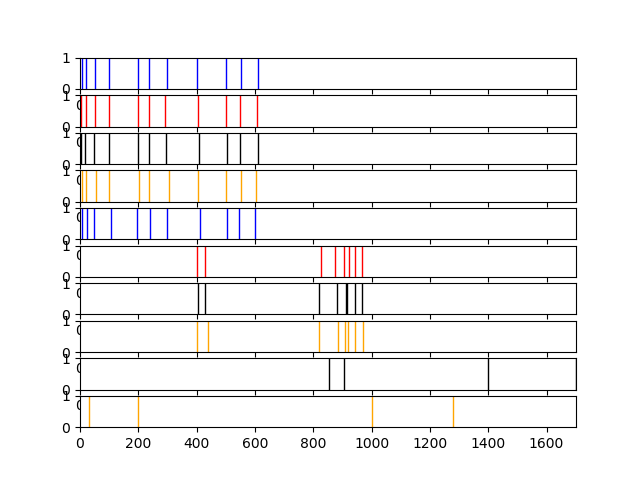}
    \caption{Time series of change points with no outliers}
    \label{fig:Collection of Time Series Change Points with (virtually) no outliers}
\end{figure}

\begin{table}[t]
\tabcolsep 2pt
\begin{tabular}{ |p{2cm}||p{.85cm}|p{.85cm}|p{.85cm}|p{.85cm}|p{.85cm}|p{.85cm}|p{.85cm}|p{.85cm}|p{.85cm}|p{.85cm}|}
 \hline
 \multicolumn{11}{|c|}{Spectral clustering results} \\
 \hline
 Metric & TS1 & TS2 & TS3 & TS4 & TS5 & TS6 & TS7 & TS8 & TS9 & TS10\\
 \hline
 %Hausdorff & 1 & 1 & 1 & 1 & 1 & 2 & 2 & 2 & 3 & 4\\
 %MH$_1$ & 2 & 2 & 2 & 2 & 2 & 1 & 1 & 1 & 4 & 3\\
 %MH$_2$ & 1 & 1 & 1 & 1 & 1 & 4 & 2 & 1 & 3 & 1\\
 %MH$_3$ & 2 & 2 & 2 & 2 & 2 & 1 & 1 & 1 & 3 & 4\\
 %Wasserstein & 4 & 2 & 3 & 1 & 2 & 2 & 2 & 2 & 2 & 2\\
 %MJ$_{0.5}$ & 1 & 1 & 1 & 1 & 1 & 2 & 2 & 2 & 4 & 3\\ 
 %MJ$_1$ & 1 & 1 & 1 & 1 & 1 & 2 & 2 & 2 & 4 & 3\\
 %MJ$_2$ & 2 & 2 & 2 & 2 & 2 & 3 & 3 & 3 & 4 & 1 \\
 \hline	
Hausdorff & 1 & 1 & 1 & 1 & 1 & 2 & 2 & 2 & 3 & 4\\
 MH$_1$ & 1 & 1 & 1 & 1 & 1 & 2 & 2 & 2 & 3 & 4\\
 MH$_2$ & 1 & 1 & 1 & 1 & 1 & 2 & 3 & 1 & 4 & 1\\
 MH$_3$ & 1 & 1 & 1 & 1 & 1 & 2 & 2 & 2 & 3 & 4\\
 Wasserstein & 1 & 2 & 3 & 4 & 2 & 2 & 2 & 2 & 2 & 2\\
 MJ$_{0.5}$ & 1 & 1 & 1 & 1 & 1 & 2 & 2 & 2 & 3 & 4\\ 
 MJ$_1$ & 1 & 1 & 1 & 1 & 1 & 2 & 2 & 2 & 3 & 4\\
 MJ$_2$ & 1 & 1 & 1 & 1 & 1 & 2 & 2 & 2 & 3 & 4 \\
\hline	 
\end{tabular}
\caption{Spectral clustering distance matrices with no outliers}
\label{tab:result_table_no}
\end{table}

Figure \ref{fig:Collection of Time Series Change Points with (virtually) no outliers} displays the ten time series of candidate change points. When assessing if two time series are similar with respect to their change points, we are interested in both the location and number of change points. There is an average spacing of about 35 units between change points; this simulates realistic outputs from a change point detection algorithm that generally requires a minimum number of data points within locally stationary segments. In this scenario, one should consider the first five time series (1-5 incl.) as similar, the next three (6-8 incl.) as similar and  the final two (9 and 10) as dissimilar to all other time series.  Although there are no change point outliers in this scenario, it is instructive to measure how various distance measures perform without the presence of outliers.

Interpreting similarity among large collections of time series' change points may be a difficult task. Therefore, we make inference using all three of our proposed methods in Section \ref{Section Distance matrix analysis} to analyse some candidate distance matrix. Perhaps the most concise and expressive display of general similarity or dissimilarity within any such collection is the plot of the absolute value of the eigenvalues of the distance matrices. We compare all our distance measures in Figure \ref{fig:No_Eigenvalue_Analysis}. All distance measures appear to indicate that there are five time series that are highly similar, three that are slightly less similar and two far more dissimilar to the rest of the collection. In this instance, all eigenvalue plots look very similar: without the existence of outliers, all these distance measures perform similarly, so this is expected. Note the difference in scale of the diagrams reflects the value of $|\lambda_n|$, hence the total scale of these matrices.

\begin{figure}[t]
    \centering
    \begin{subfigure}[b]{0.24\textwidth}
        \includegraphics[width=\textwidth]{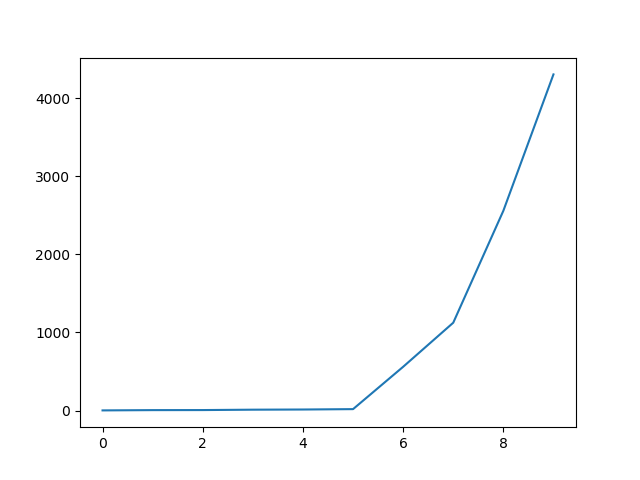}
        \caption{Hausdorff}
        \label{fig:NoHausdorffEigenvalues}
    \end{subfigure}
        \begin{subfigure}[b]{0.24\textwidth}
        \includegraphics[width=\textwidth]{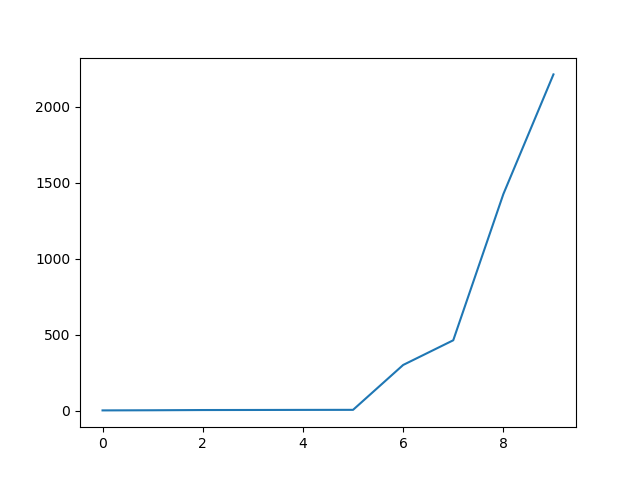}
        \caption{MH$_1$}
        \label{fig:No_MH1_Eigenvalues}
    \end{subfigure}
        \begin{subfigure}[b]{0.24\textwidth}
        \includegraphics[width=\textwidth]{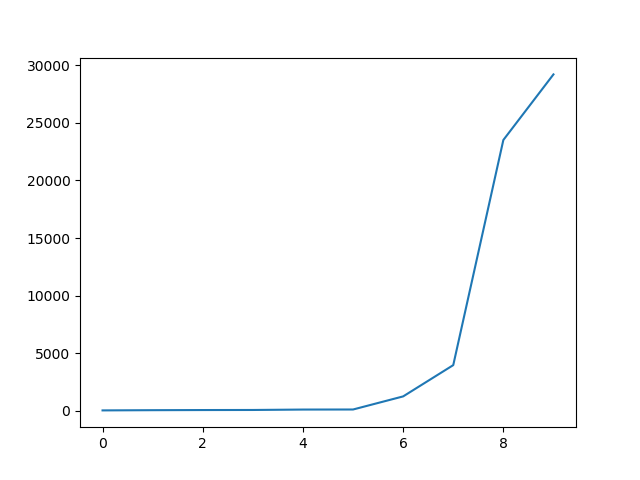}
        \caption{MH$_2$}
        \label{fig:No_MH2_Eigenvalues}
    \end{subfigure}
        \begin{subfigure}[b]{0.24\textwidth}
        \includegraphics[width=\textwidth]{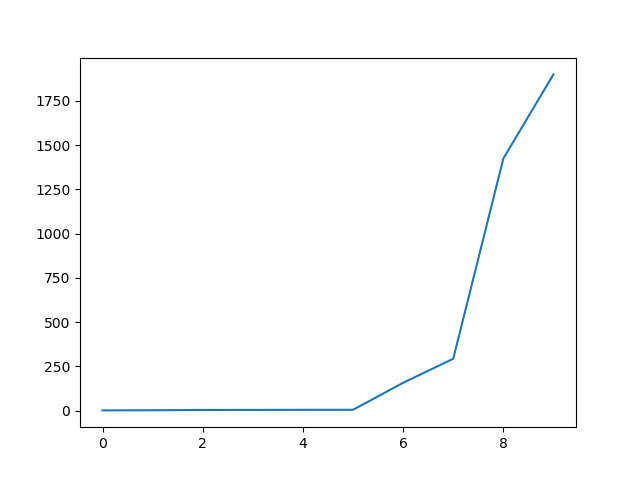}
        \caption{MH$_3$}
        \label{fig:No_MH3_Eigenvalues}
    \end{subfigure}
    \begin{subfigure}[b]{0.24\textwidth}
        \includegraphics[width=\textwidth]{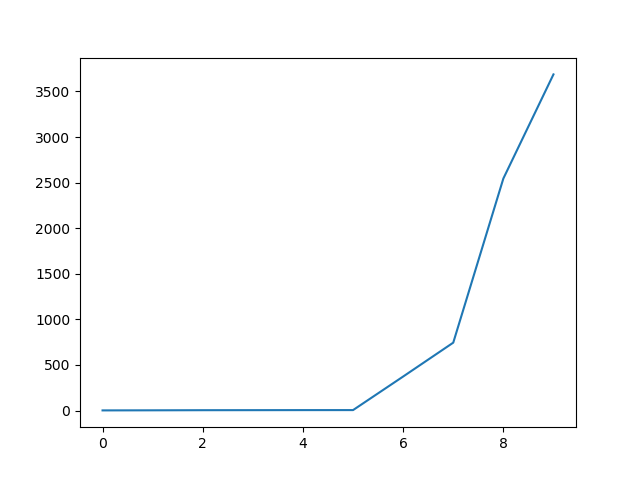}
        \caption{Wasserstein}
        \label{fig:No_Wasserstein_Eigenvalues}
    \end{subfigure}
    \begin{subfigure}[b]{0.24\textwidth}
        \includegraphics[width=\textwidth]{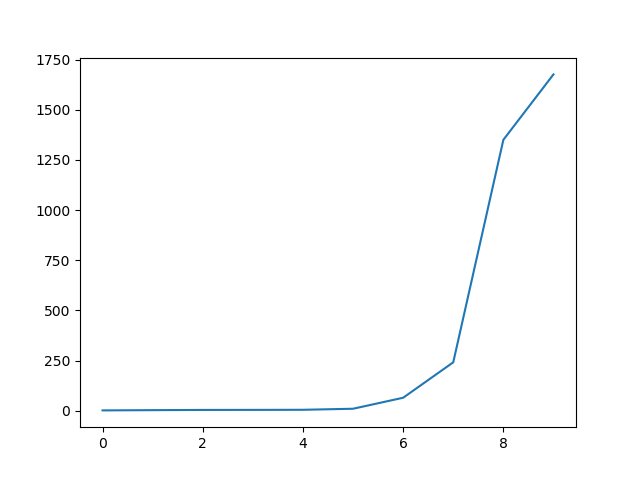}
        \caption{MJ$_{0.5}$}
        \label{fig:No_MJ05_Eigenvalues}
    \end{subfigure}
    \begin{subfigure}[b]{0.24\textwidth}
        \includegraphics[width=\textwidth]{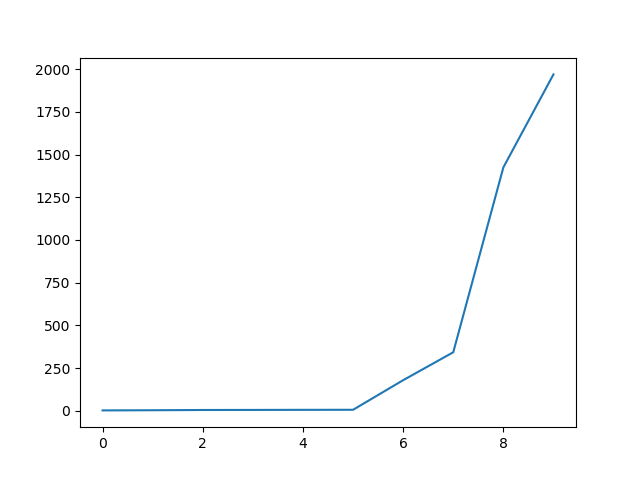}
        \caption{MJ$_1$}
        \label{fig:No_MJ1_Eigenvalues}
    \end{subfigure}
    \begin{subfigure}[b]{0.24\textwidth}
        \includegraphics[width=\textwidth]{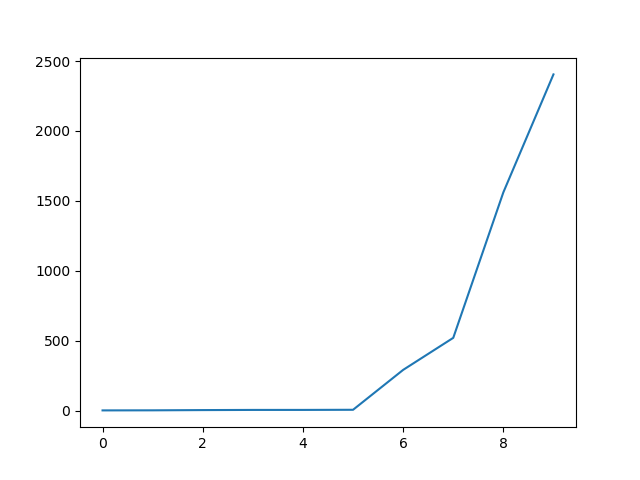}
        \caption{MJ$_2$}
        \label{fig:No_MJ2_Eigenvalues}
    \end{subfigure}
    \caption{Eigenvalue analysis with no outliers}\label{fig:No_Eigenvalue_Analysis}
\end{figure}

Table \ref{tab:result_table_no} shows that of our eight distance measures, six distance measures, namely Hausdorff, MH$_1$, MH$_3$, MJ$_{0.5}$, MJ$_1$ and MJ$_2$, cluster the time series correctly. Both the Wasserstein and MH$_2$ distances fail to determine appropriate clusters within the spectral clustering. The dendrograms in Figure \ref{fig:No_Dendrogram_Analysis} should be analysed carefully. All distance measures indicate that there is a cluster of five change point sets that are similar, another cluster of three and two unrelated change point sets. However, spectral clustering highlights that the Wasserstein and MH$_2$ distance measures incorrectly identify which time series should be considered similar. 

\begin{figure}[t]
    \centering
    \begin{subfigure}[b]{0.24\textwidth}
        \includegraphics[width=\textwidth]{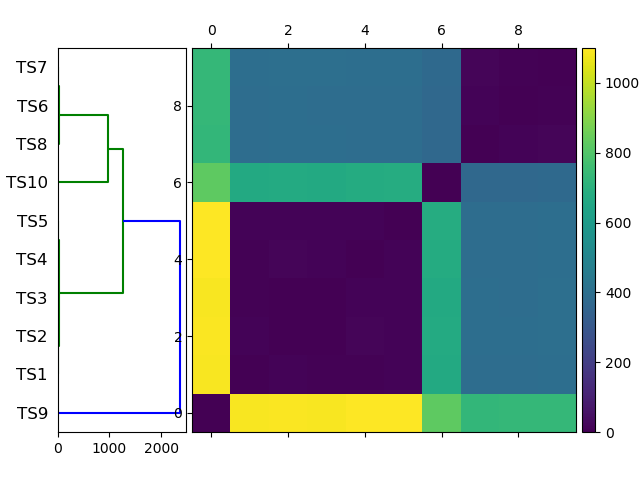}
        \caption{Hausdorff}
        \label{fig:NoHausdorffDendrogram}
    \end{subfigure}
        \begin{subfigure}[b]{0.24\textwidth}
        \includegraphics[width=\textwidth]{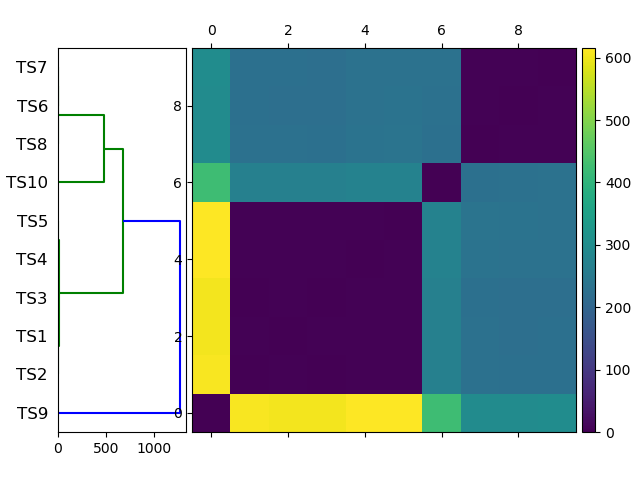}
        \caption{MH$_1$}
        \label{fig:No_MH1_Dendrogram}
    \end{subfigure}
        \begin{subfigure}[b]{0.24\textwidth}
        \includegraphics[width=\textwidth]{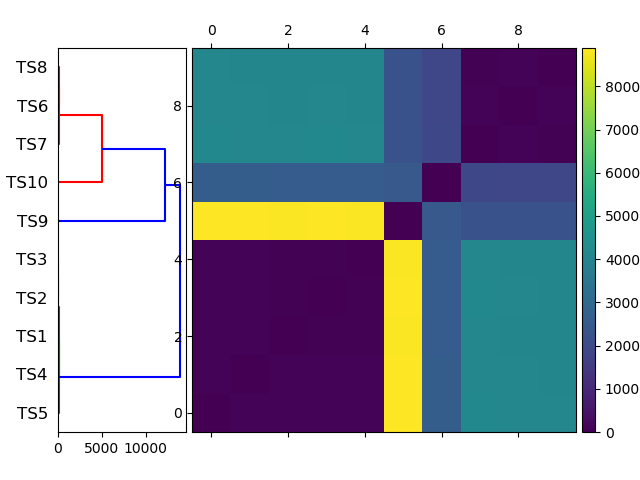}
        \caption{MH$_2$}
        \label{fig:No_MH2_Dendrogram}
    \end{subfigure}
        \begin{subfigure}[b]{0.24\textwidth}
        \includegraphics[width=\textwidth]{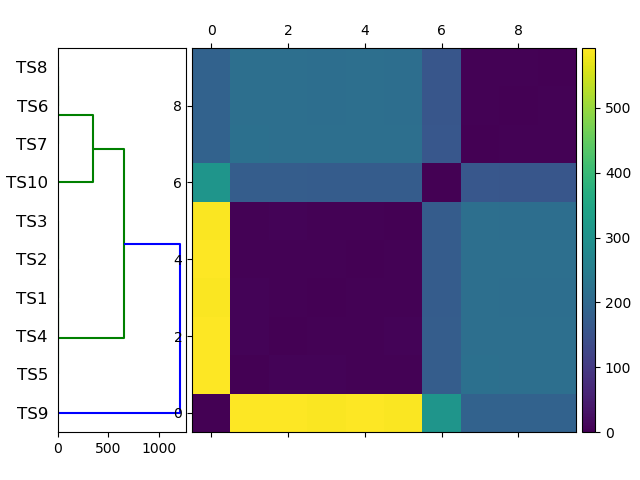}
        \caption{MH$_3$}
        \label{fig:No_MH3_Dendrogram}
    \end{subfigure}
    \begin{subfigure}[b]{0.24\textwidth}
        \includegraphics[width=\textwidth]{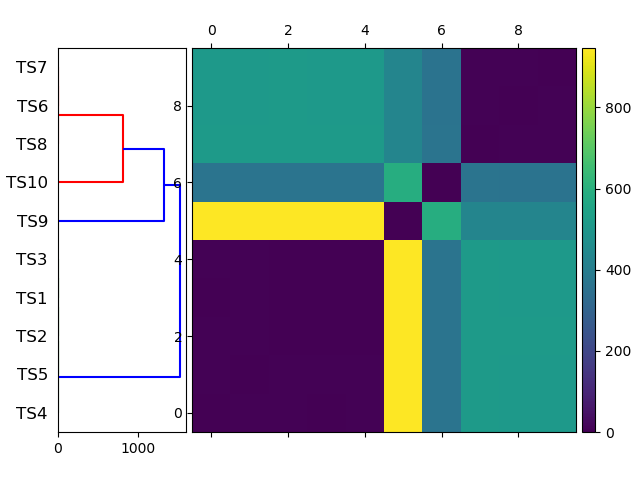}
        \caption{Wasserstein}
        \label{fig:No_Wasserstein_Dendrogram}
    \end{subfigure}
    \begin{subfigure}[b]{0.24\textwidth}
        \includegraphics[width=\textwidth]{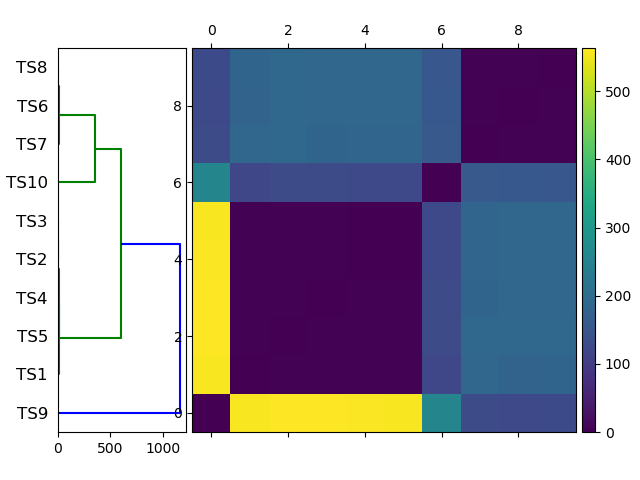}
        \caption{MJ$_{0.5}$}
        \label{fig:No_MJ05_Dendrogram}
    \end{subfigure}
        \begin{subfigure}[b]{0.24\textwidth}
        \includegraphics[width=\textwidth]{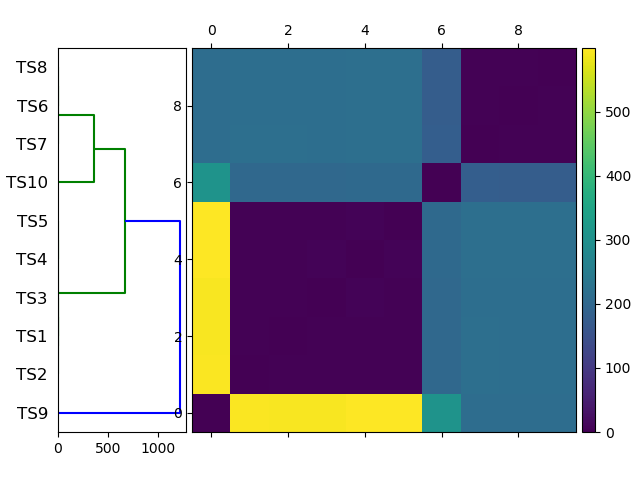}
        \caption{MJ$_1$}
        \label{fig:No_MJ1_Dendrogram}
    \end{subfigure}
    \begin{subfigure}[b]{0.24\textwidth}
        \includegraphics[width=\textwidth]{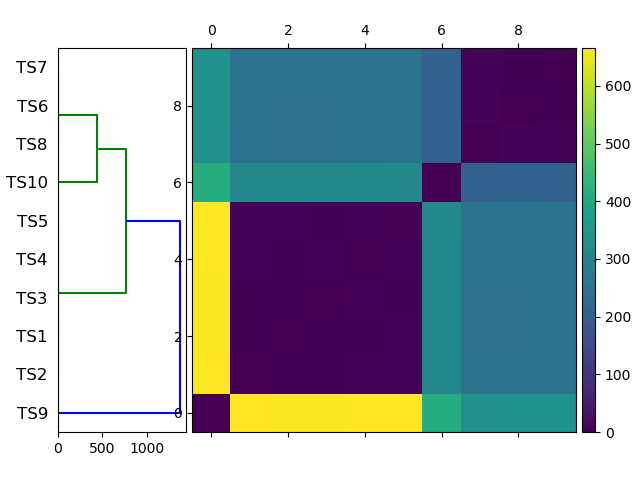}
        \caption{MJ$_2$ }
        \label{fig:No_MJ2_Dendrogram}
    \end{subfigure}
    \caption{Dendrogram analysis with no outliers}\label{fig:No_Dendrogram_Analysis}
\end{figure} 

We also analyse the transitivity, described in Section \ref{Transitivity analysis}, over all sets of change points for the eight semi-metrics in our analysis. As seen in Figure \ref{fig:No_Transitivity_Analysis}, the MJ$_{0.5}$ fails most significantly, with 51\% of potential triples failing, and an average fail ratio of 1.28. Both the MJ$_1$ and MJ$_2$ distances fail the triangle inequality in this simulation, with 4\% of elements in the matrix failing for the MJ$_1$ distance and 3\% of of elements failing for the MJ$_2$ distance. Of those elements that fail the triangle inequality, the average fail ratio is 1.32 and 1.14 for the MJ$_1$ and MJ$_2$ distances respectively. As expected, as $p$ increases, transitivity seems to improve. All of the modified Hausdorff distances fail the triangle inequality too. 8.5\% of MH$_1$ triples fail the triangle inequality, with an average fail ratio of 1.06. The MH$_2$ has a lower percentage of failed triples than the MH$_1$ with only 4\% failing, however those that do fail perform significantly worse, with an average fail ratio of 1.49. The MH$_3$ also has 4\% of triples fail, with a less severe average fail ratio of 1.41. So in this scenario the MJ$_{0.5}$ has the most failed triples by a significant margin, however the MH$_2$ has the highest fail ratio. This shows that MH$_2$ violates the triangle inequality most severely. The MJ$_1$ and MJ$_2$ perform better with the triangle inequality than the MH$_1$.

\begin{figure}[t]
    \centering
    \begin{subfigure}[b]{0.24\textwidth}
        \includegraphics[width=\textwidth]{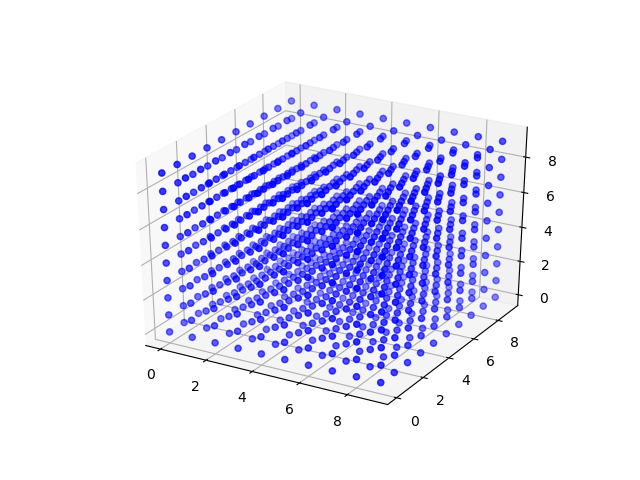}
        \caption{Hausdorff}
        \label{fig:NoHausdorffTransitivity}
    \end{subfigure}
        \begin{subfigure}[b]{0.24\textwidth}
        \includegraphics[width=\textwidth]{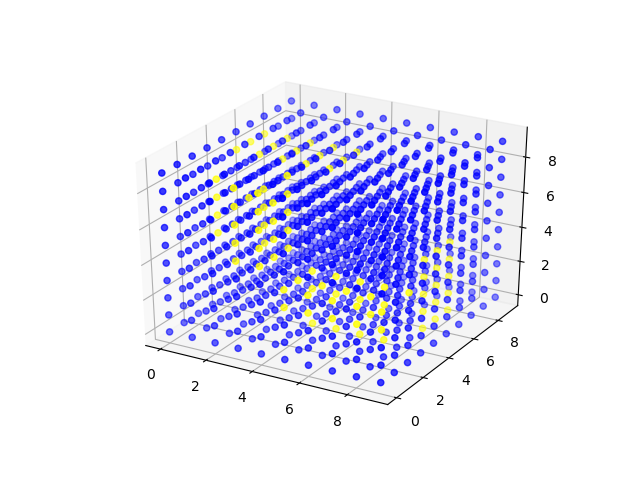}
        \caption{MH$_1$}
        \label{fig:No_MH1_Transitivity}
    \end{subfigure}
        \begin{subfigure}[b]{0.24\textwidth}
        \includegraphics[width=\textwidth]{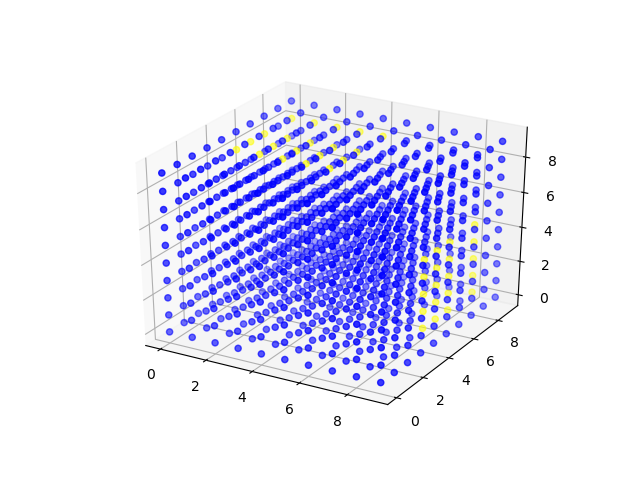}
        \caption{MH$_2$}
        \label{fig:No_MH2_Transitivity}
    \end{subfigure}
        \begin{subfigure}[b]{0.24\textwidth}
        \includegraphics[width=\textwidth]{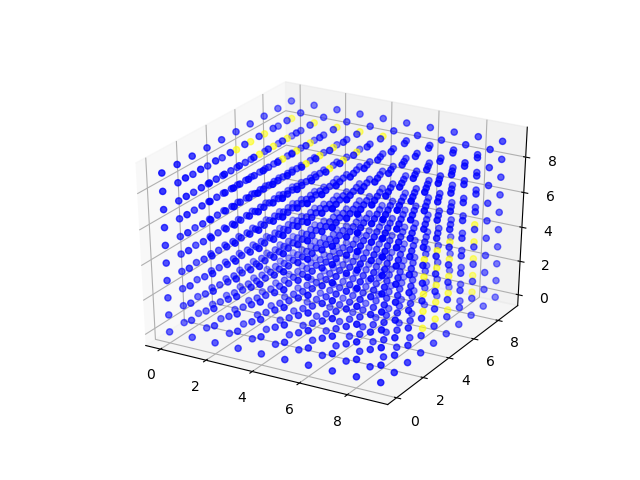}
        \caption{MH$_3$}
        \label{fig:No_MH3_Transitivity}
    \end{subfigure}
    \begin{subfigure}[b]{0.24\textwidth}
        \includegraphics[width=\textwidth]{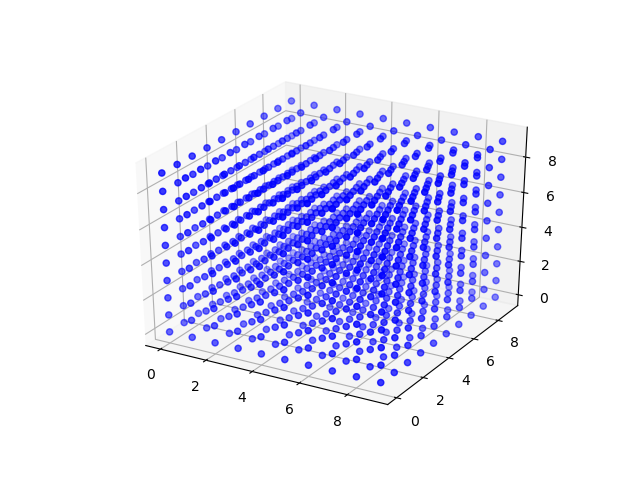}
        \caption{Wasserstein}
        \label{fig:No_Wasserstein_Transitivity}
    \end{subfigure}
    \begin{subfigure}[b]{0.24\textwidth}
        \includegraphics[width=\textwidth]{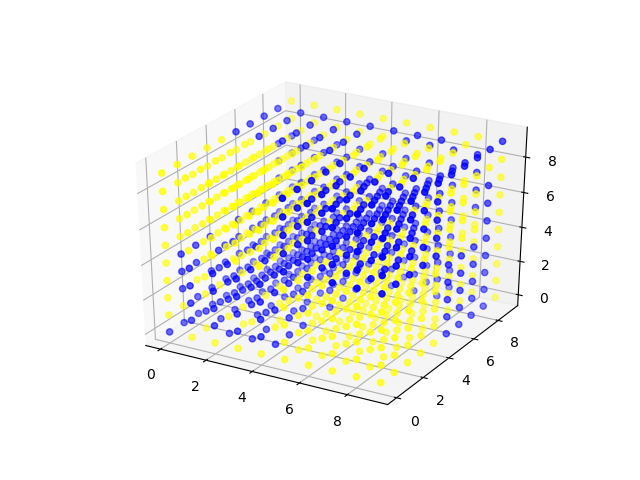}
        \caption{MJ$_{0.5}$}
        \label{fig:No_MJ05_Transitivity}
    \end{subfigure}
        \begin{subfigure}[b]{0.24\textwidth}
        \includegraphics[width=\textwidth]{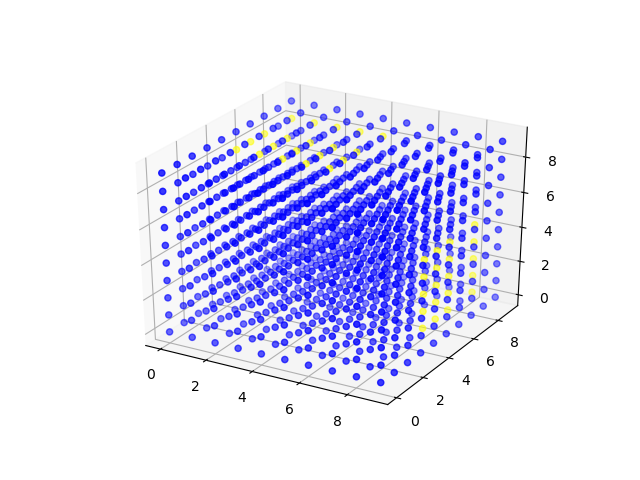}
        \caption{MJ$_1$}
        \label{fig:No_MJ1_Transitivity}
    \end{subfigure}
    \begin{subfigure}[b]{0.24\textwidth}
        \includegraphics[width=\textwidth]{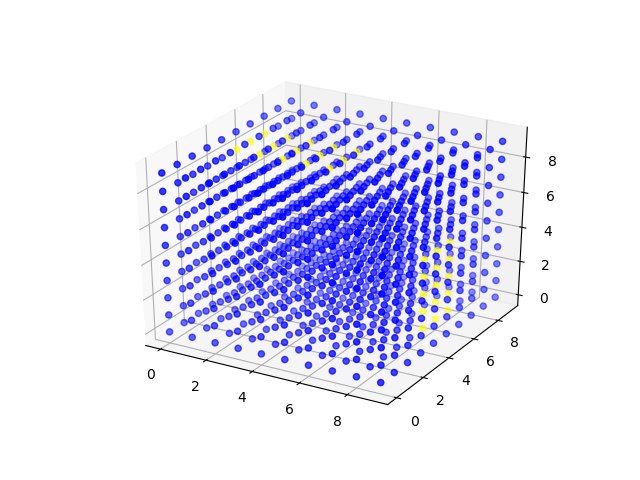}
        \caption{MJ$_2$ }
        \label{fig:No_MJ2_Transitivity}
    \end{subfigure}
    \caption{Transitivity analysis with no outliers}\label{fig:No_Transitivity_Analysis}
\end{figure}

\subsection{Simulation $2$: moderate change point outliers}

\begin{figure}[h]
    \centering
    \includegraphics[width=.48\textwidth]{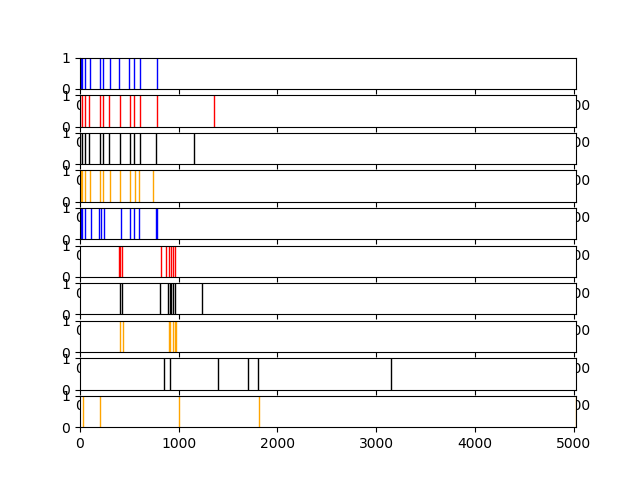}
    \caption{Time series of change points with moderate outliers}
    \label{fig:Collection of Time Series Change Points with several outliers}
\end{figure}

\begin{table}[h]
\tabcolsep 2pt
\begin{tabular}{ |p{2cm}||p{.85cm}|p{.85cm}|p{.85cm}|p{.85cm}|p{.85cm}|p{.85cm}|p{.85cm}|p{.85cm}|p{.85cm}|p{.85cm}|}
 \hline
 \multicolumn{11}{|c|}{Spectral clustering results} \\
 \hline
 Metric & TS1 & TS2 & TS3 & TS4 & TS5 & TS6 & TS7 & TS8 & TS9 & TS10\\
 \hline
% Hausdorff & 3 & 1 & 1 & 3 & 3 & 1 & 1 & 1 & 4 & 2\\
% MH$_1$ & 2 & 2 & 2 & 2 & 2 & 3 & 3 & 3 & 1 & 4\\
% MH$_2$ & 1 & 1 & 1 & 1 & 1 & 2 & 2 & 2 & 3 & 4\\
% MH$_3$ & 3 & 3 & 3 & 3 & 3 & 1 & 1 & 1 & 2 & 4\\
% Wasserstein & 4 & 1 & 3 & 1 & 2 & 1 & 1 & 1 & 1 & 1\\
% MJ$_{0.5}$ & 1 & 1 & 1 & 1 & 1 & 2 & 2 & 2 & 4 & 3\\
% MJ$_1$ & 1 & 1 & 1 & 1 & 1 & 2 & 2 & 2 & 4 & 3\\
% MJ$_2$ & 1 & 1 & 1 & 1 & 1 & 2 & 2 & 2 & 4 & 3 \\
% \hline	
 Hausdorff & 1 & 2 & 2 & 1 & 1 & 2 & 2 & 2 & 3 & 4\\
 MH$_1$ & 1 & 1 & 1 & 1 & 1 & 2 & 2 & 2 & 3 & 4\\
 MH$_2$ & 1 & 1 & 1 & 1 & 1 & 2 & 2 & 2 & 3 & 4\\
 MH$_3$ & 1 & 1 & 1 & 1 & 1 & 2 & 2 & 2 & 3 & 4\\
 Wasserstein & 1 & 2 & 3 & 2 & 4 & 2 & 2 & 2 & 2 & 2\\
 MJ$_{0.5}$ & 1 & 1 & 1 & 1 & 1 & 2 & 2 & 2 & 3 & 4\\
 MJ$_1$ & 1 & 1 & 1 & 1 & 1 & 2 & 2 & 2 & 3 & 4\\
 MJ$_2$ & 1 & 1 & 1 & 1 & 1 & 2 & 2 & 2 & 3 & 4 \\
 \hline	
\end{tabular}
\caption{Spectral clustering distance matrices with moderate outliers}
\label{tab:result_table_moderate}
\end{table}

Figure \ref{fig:Collection of Time Series Change Points with several outliers} shows ten simulated time series change points that exhibit outliers with moderate frequency and severity. This is a more realistic scenario than Figure \ref{fig:Collection of Time Series Change Points with (virtually) no outliers}, as outliers occur regularly when applying change point detection algorithms to multiple time series. Again, time series 1-5, 6-8, 9-10 should be identified as separate clusters.

Figure \ref{fig:Several_Eigenvalue_Analysis} displays the increasing absolute value of the eigenvalues. In this simulation, we see that that Hausdorff distance in Figure \ref{fig:SeveralHausdorffEigenvalues} has detected eight time series that are highly similar, and two that are dissimilar relative to the others. Other distance measures have identified the general similarity more appropriately. The MH$_1$, MH$_3$, MJ$_{0.5}$, MJ$_1$ and MJ$_2$ metrics in particular appear to produce sensible outputs.

\begin{figure}[h]
    \centering
    \begin{subfigure}[b]{0.24\textwidth}
        \includegraphics[width=\textwidth]{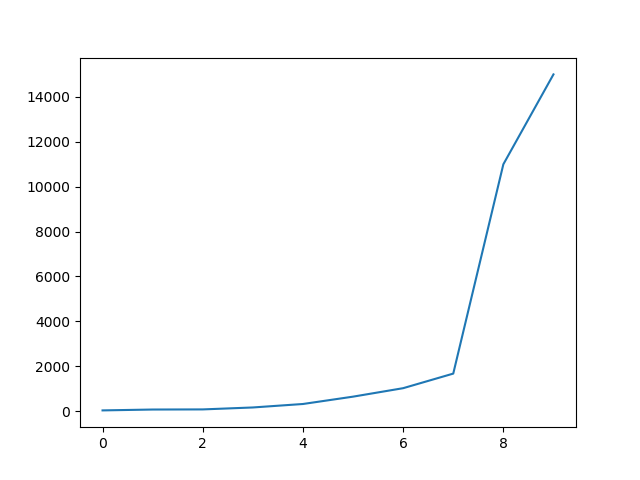}
        \caption{Hausdorff}
        \label{fig:SeveralHausdorffEigenvalues}
    \end{subfigure}
        \begin{subfigure}[b]{0.24\textwidth}
        \includegraphics[width=\textwidth]{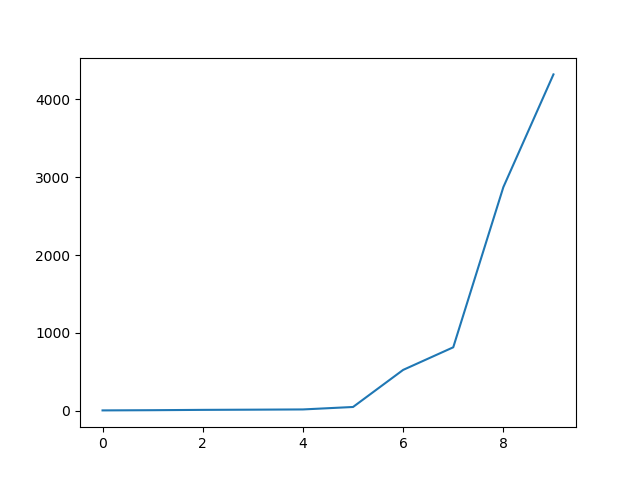}
        \caption{MH$_1$}
        \label{fig:Several_MH1_Eigenvalues}
    \end{subfigure}
        \begin{subfigure}[b]{0.24\textwidth}
        \includegraphics[width=\textwidth]{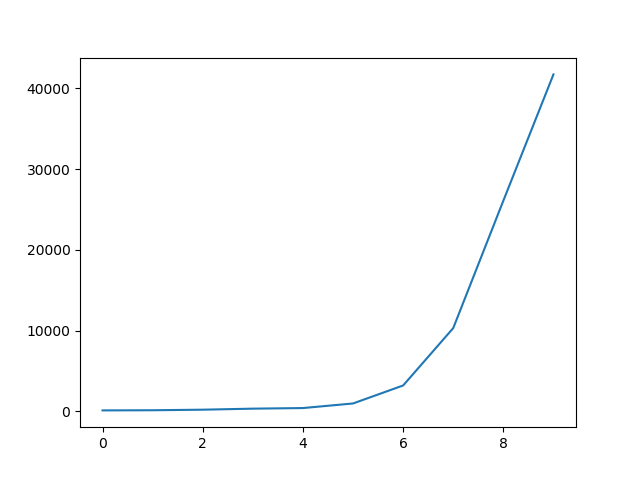}
        \caption{MH$_2$}
        \label{fig:Several_MH2_Eigenvalues}
    \end{subfigure}
        \begin{subfigure}[b]{0.24\textwidth}
        \includegraphics[width=\textwidth]{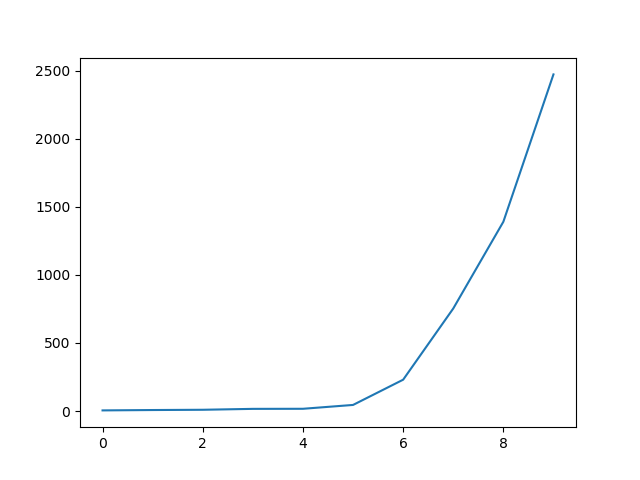}
        \caption{MH$_3$}
        \label{fig:Several_MH3_Eigenvalues}
    \end{subfigure}
    \begin{subfigure}[b]{0.24\textwidth}
        \includegraphics[width=\textwidth]{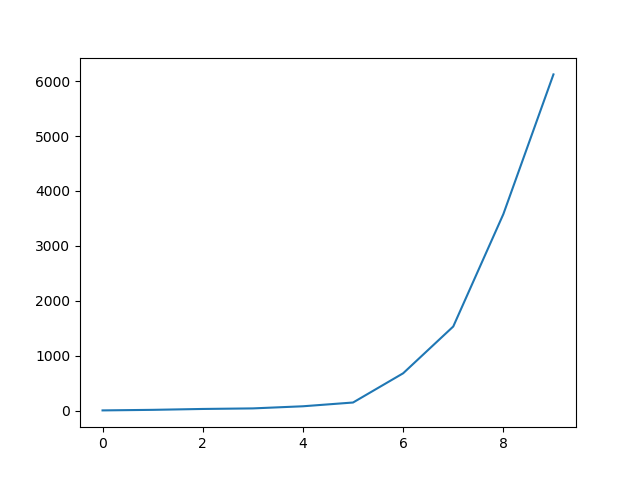}
        \caption{Wasserstein}
        \label{fig:Several_Wasserstein_Eigenvalues}
    \end{subfigure}
    \begin{subfigure}[b]{0.24\textwidth}
        \includegraphics[width=\textwidth]{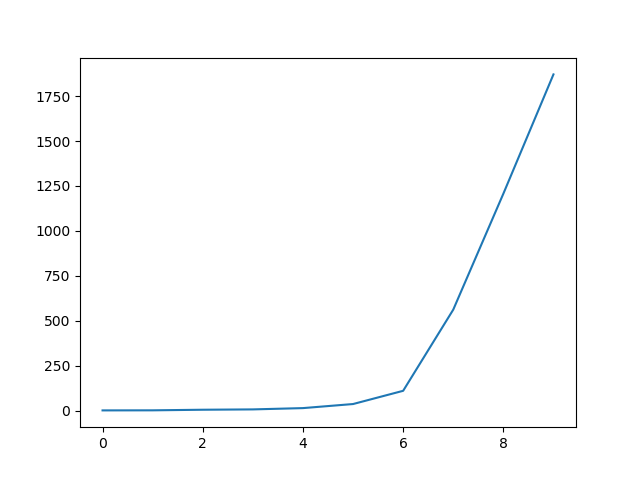}
        \caption{MJ$_{0.5}$}
        \label{fig:Several_MJ05_Eigenvalues}
    \end{subfigure}
        \begin{subfigure}[b]{0.24\textwidth}
        \includegraphics[width=\textwidth]{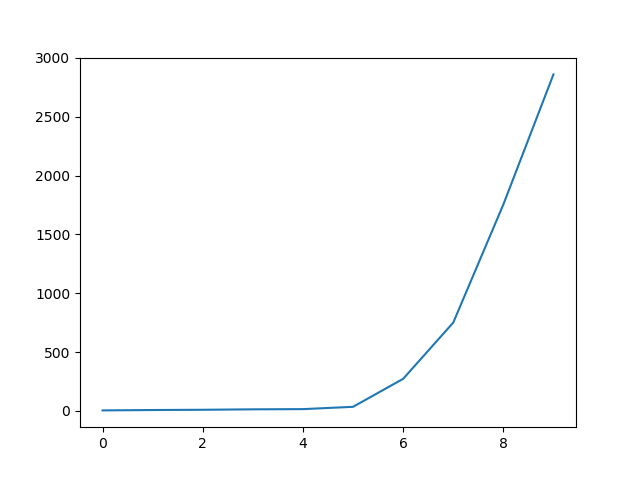}
        \caption{MJ$_1$}
        \label{fig:Several_MJ1_Eigenvalues}
    \end{subfigure}
    \begin{subfigure}[b]{0.24\textwidth}
        \includegraphics[width=\textwidth]{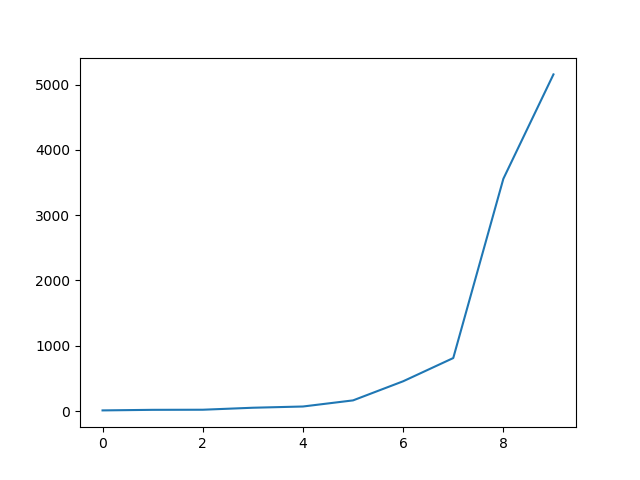}
        \caption{MJ$_2$}
        \label{fig:Several_MJ2_Eigenvalues}
    \end{subfigure}
    \caption{Eigenvalue analysis with moderate outliers}
    \label{fig:Several_Eigenvalue_Analysis}
\end{figure}

\begin{figure}[h]
    \centering
    \begin{subfigure}[b]{0.24\textwidth}
        \includegraphics[width=\textwidth]{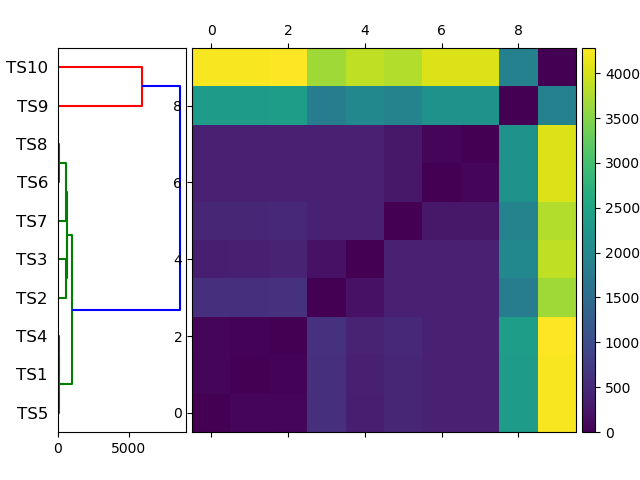}
        \caption{Hausdorff}
        \label{fig:SeveralHausdorffDendrogram}
    \end{subfigure}
        \begin{subfigure}[b]{0.24\textwidth}
        \includegraphics[width=\textwidth]{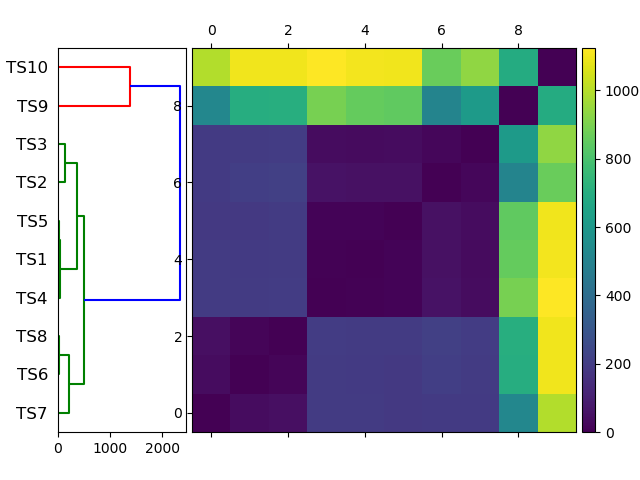}
        \caption{MH$_1$}
        \label{fig:Several_MH1_Dendrogram}
    \end{subfigure}
        \begin{subfigure}[b]{0.24\textwidth}
        \includegraphics[width=\textwidth]{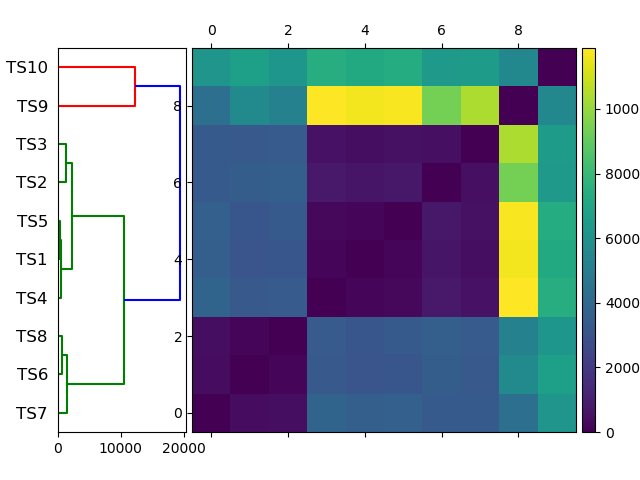}
        \caption{MH$_2$}
        \label{fig:Several_MH2_Dendrogram}
    \end{subfigure}
        \begin{subfigure}[b]{0.24\textwidth}
        \includegraphics[width=\textwidth]{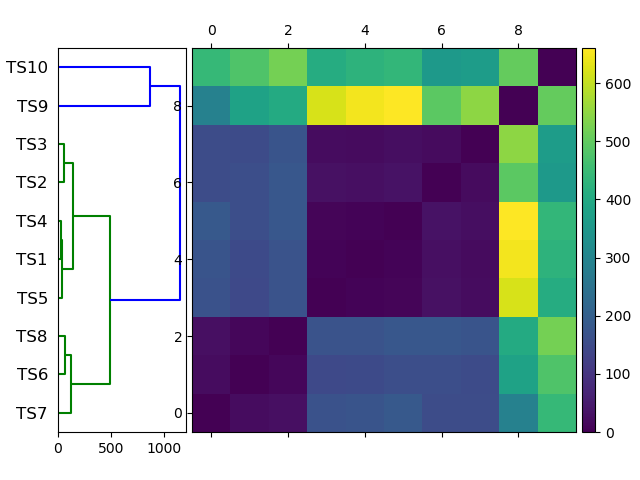}
        \caption{MH$_3$}
        \label{fig:Several_MH3_Dendrogram}
    \end{subfigure}
    \begin{subfigure}[b]{0.24\textwidth}
        \includegraphics[width=\textwidth]{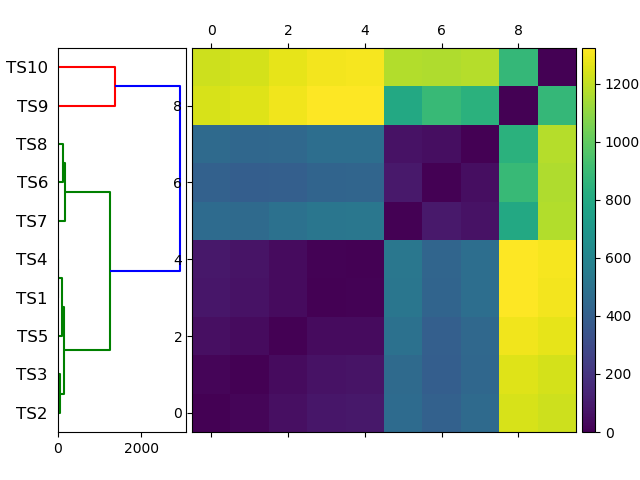}
        \caption{Wasserstein}
        \label{fig:Several_Wasserstein_Dendrogram}
    \end{subfigure}
    \begin{subfigure}[b]{0.24\textwidth}
        \includegraphics[width=\textwidth]{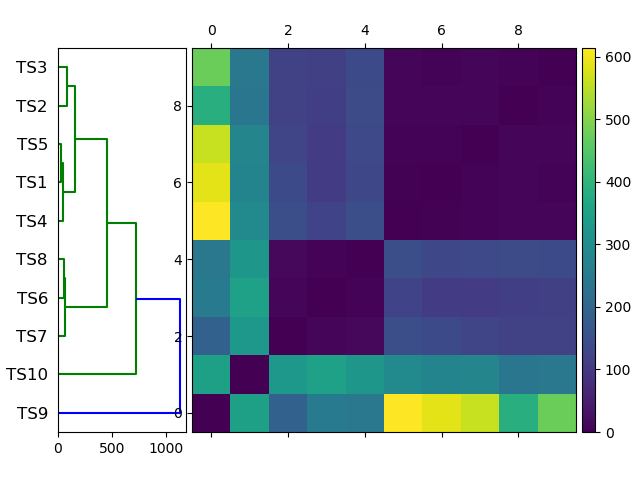}
        \caption{MJ$_{0.5}$}
        \label{fig:Several_MJ05_Dendrogram}
    \end{subfigure}
        \begin{subfigure}[b]{0.24\textwidth}
        \includegraphics[width=\textwidth]{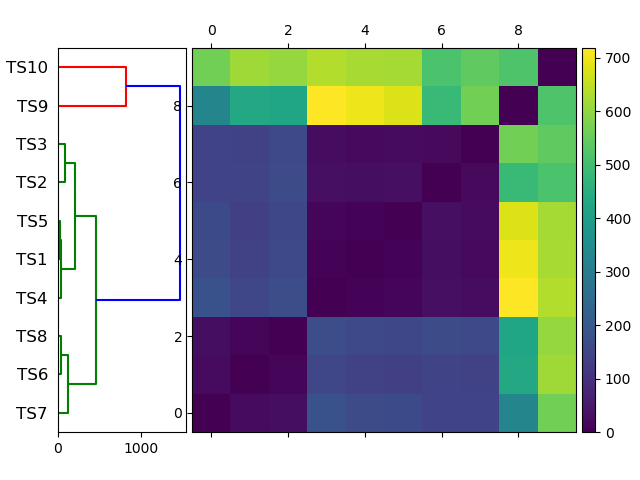}
        \caption{MJ$_1$}
        \label{fig:Several_MJ1_Dendrogram}
    \end{subfigure}
    \begin{subfigure}[b]{0.24\textwidth}
        \includegraphics[width=\textwidth]{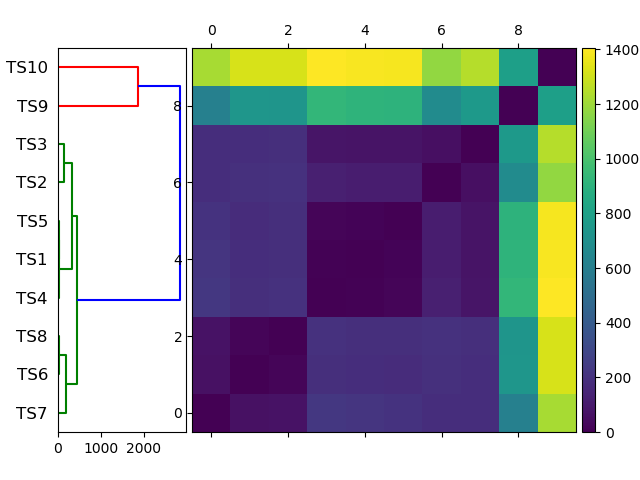}
        \caption{MJ$_2$}
        \label{fig:Several_MJ2_Dendrogram}
    \end{subfigure}
    \caption{Dendrogram analysis with moderate outliers}\label{fig:Several_Dendrogram_Analysis}
\end{figure}

The dendrograms displayed in Figure \ref{fig:Several_Dendrogram_Analysis} illustrate the Hausdorff distance's sensitivity to outliers. The remaining six distance measures correctly identify the general structure in the time series collection. That is, there are two separate clusters of highly similar time series and two unrelated time series. 

The spectral clustering results in Table \ref{tab:result_table_moderate} indicate that five of the seven distance measures correctly identified similar groupings of time series, namely, MH$_1$, MH$_2$, MH$_3$, MJ$_{0.5}$, MJ$_1$ and MJ$_2$ produced the correct groupings of time series. Once again, the Wasserstein distance, although producing eigenvalue and dendrogram outputs consistent with the successful distance measures, proposed inappropriate collections of time series within candidate clusters.

\begin{figure}[h]
    \centering
    \begin{subfigure}[b]{0.24\textwidth}
        \includegraphics[width=\textwidth]{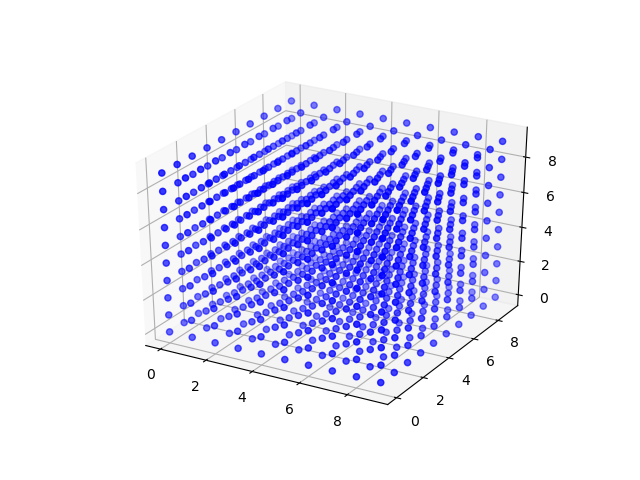}
        \caption{Hausdorff}
        \label{fig:SeveralHausdorffTransitivity}
    \end{subfigure}
        \begin{subfigure}[b]{0.24\textwidth}
        \includegraphics[width=\textwidth]{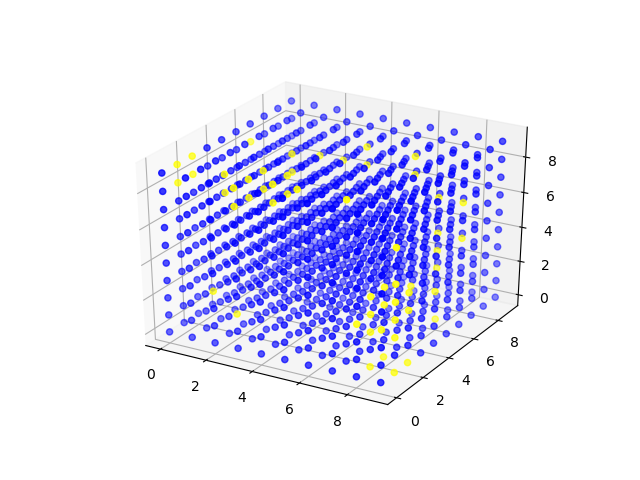}
        \caption{MH$_1$}
        \label{fig:Several_MH1_Transitivity}
    \end{subfigure}
        \begin{subfigure}[b]{0.24\textwidth}
        \includegraphics[width=\textwidth]{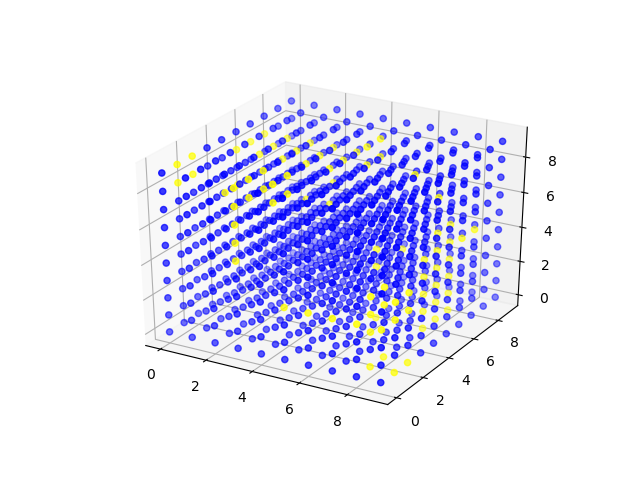}
        \caption{MH$_2$}
        \label{fig:Several_MH2_Transitivity}
    \end{subfigure}
        \begin{subfigure}[b]{0.24\textwidth}
        \includegraphics[width=\textwidth]{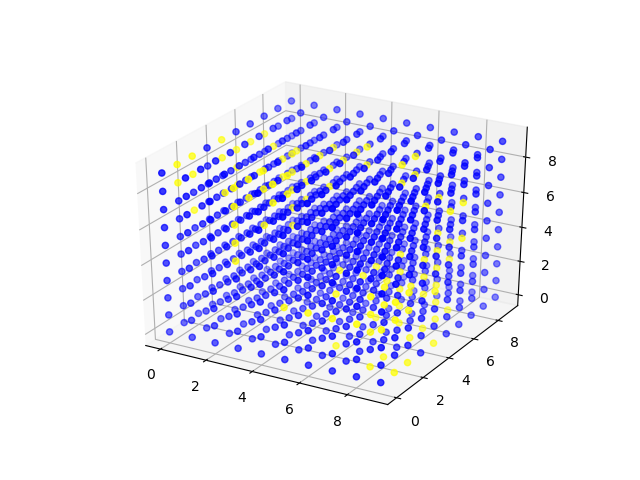}
        \caption{MH$_3$}
        \label{fig:Several_MH3_Transitivity}
    \end{subfigure}
    \begin{subfigure}[b]{0.24\textwidth}
        \includegraphics[width=\textwidth]{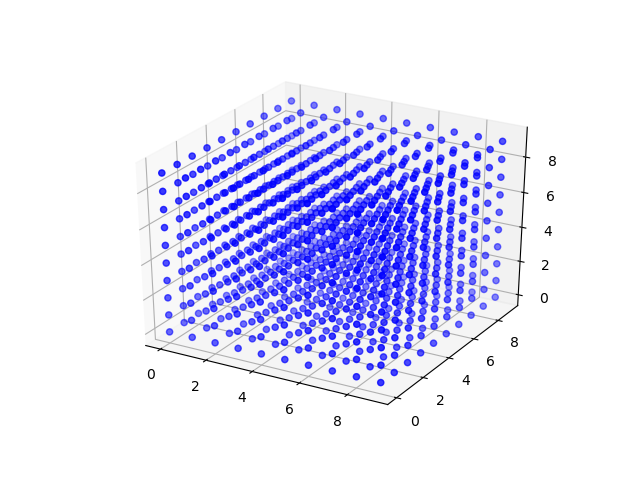}
        \caption{Wasserstein}
        \label{fig:Several_Wasserstein_Transitivity}
    \end{subfigure}
    \begin{subfigure}[b]{0.24\textwidth}
        \includegraphics[width=\textwidth]{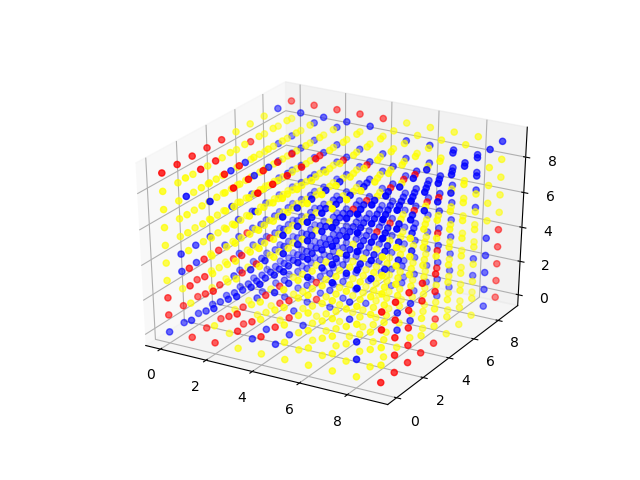}
        \caption{MJ$_{0.5}$}
        \label{fig:Several_MJ05_Transitivity}
    \end{subfigure}
        \begin{subfigure}[b]{0.24\textwidth}
        \includegraphics[width=\textwidth]{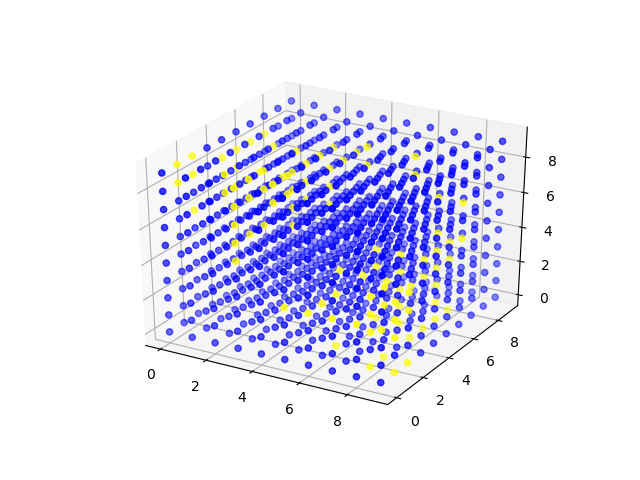}
        \caption{MJ$_1$}
        \label{fig:Several_MJ1_Transitivity}
    \end{subfigure}
    \begin{subfigure}[b]{0.24\textwidth}
        \includegraphics[width=\textwidth]{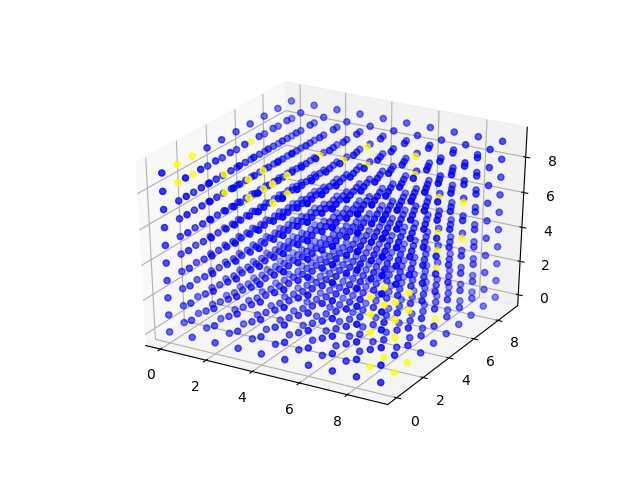}
        \caption{MJ$_2$}
        \label{fig:Several_MJ2_Transitivity}
    \end{subfigure}
    \caption{Transitivity analysis with moderate outliers}\label{fig:Several_Transitivity_Analysis}
\end{figure}

In the presence of moderate outliers, the MJ$_{0.5}$ violates the triangle inequality worst, both in terms of the percentage of failed triples of 58\% and the average fail ratio of 1.54. Both the MJ$_1$ and MJ$_2$ distances fail the triangle inequality (Figure \ref{fig:Several_Transitivity_Analysis}) too, with 10\% of potential distances failing the triangle inequality when using the MJ$_1$ distance and 5.6\% with the MJ$_2$ distance. The average fail ratio is 1.13 for the MJ$_1$ distance and 1.08 for the MJ$_2$ distance. For the modified Hausdorff distances, the MH$_1$ has the worst failed triple ratio (1.19) with 5.6\% of potential triples violating the triangle inequality.  For MH$_2$, 8.4\% of triples fail the triangle inequality, with an average fail of 1.14. For MH$_3$, 10.4\% fail with an average fail ratio of 1.11. In this scenario, the MJ$_{0.5}$ is the only semi-metric whose average fail ratio and percentage of failed triples may make the ratio unusable. This may however depend on the context of usage. 

\subsection{Simulation $3$: extreme change point outliers}

\begin{figure}[t]
    \centering
    \includegraphics[width=.48\textwidth]{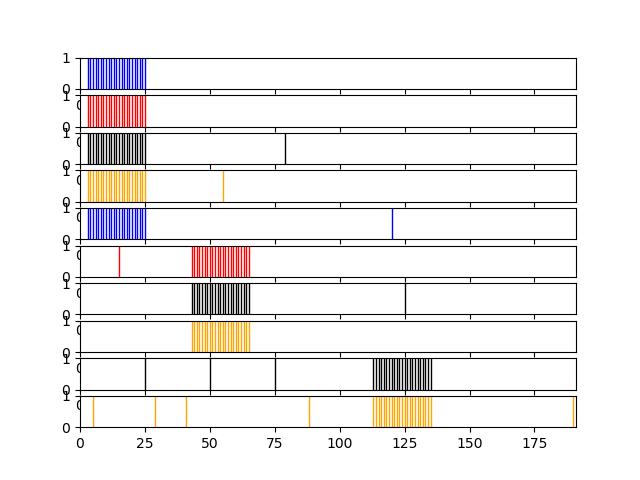}
    \caption{Time series of change points extreme outliers}
    \label{fig:Collection of Time Series Change Points with several extreme outliers}
\end{figure}

\begin{table}[t]
\tabcolsep 2pt
\begin{tabular}{ |p{2cm}||p{.85cm}|p{.85cm}|p{.85cm}|p{.85cm}|p{.85cm}|p{.85cm}|p{.85cm}|p{.85cm}|p{.85cm}|p{.85cm}|}
 \hline
 \multicolumn{11}{|c|}{Spectral clustering results} \\
 \hline
 Metric & TS1 & TS2 & TS3 & TS4 & TS5 & TS6 & TS7 & TS8 & TS9 & TS10\\
 \hline
% Hausdorff & 1 & 1 & 2 & 2 & 2 & 2 & 2 & 2 & 3 & 4\\
% MH$_1$ & 1 & 1 & 1 & 1 & 1 & 2 & 2 & 2 & 3 & 4\\
% MH$_2$ & 2 & 2 & 2 & 2 & 2 & 4 & 4 & 4 & 1 & 3\\
% MH$_3$ & 1 & 1 & 1 & 1 & 1 & 2 & 2 & 2 & 4 & 3\\
% Wasserstein & 2 & 2 & 1 & 3 & 2 & 4 & 2 & 2 & 2 & 2\\
% MJ$_{0.5}$ & 1 & 1 & 1 & 1 & 1 & 2 & 2 & 2 & 4 & 3\\
% MJ$_1$ & 1 & 1 & 1 & 1 & 1 & 2 & 2 & 2 & 4 & 3\\
% MJ$_2$ & 3 & 3 & 3 & 3 & 3 & 2 & 2 & 2 & 1 & 4 \\
% \hline	
 Hausdorff & 1 & 1 & 2 & 2 & 2 & 2 & 2 & 2 & 3 & 4\\
 MH$_1$ & 1 & 1 & 1 & 1 & 1 & 2 & 2 & 2 & 3 & 4\\
 MH$_2$ & 1 & 1 & 1 & 1 & 1 & 2 & 2 & 2 & 3 & 4\\
 MH$_3$ & 1 & 1 & 1 & 1 & 1 & 2 & 2 & 2 & 4 & 3\\
 Wasserstein & 1 & 1 & 2 & 3 & 1 & 4 & 1 & 1 & 1 & 1\\
 MJ$_{0.5}$ & 1 & 1 & 1 & 1 & 1 & 2 & 2 & 2 & 3 & 4\\
 MJ$_1$ & 1 & 1 & 1 & 1 & 1 & 2 & 2 & 2 & 3 & 4\\
 MJ$_2$ & 1 & 1 & 1 & 1 & 1 & 2 & 2 & 2 & 3 & 4 \\
 \hline	 
\end{tabular}
\caption{Spectral clustering distance matrices with extreme outliers}
\label{tab:result_table_worst}
\end{table}

Figure \ref{fig:Collection of Time Series Change Points with several extreme outliers} displays the third collection of time series' change points, where we analyse the effects of extreme outliers on candidate distance measures. This is a highly contrived scenario in our application of change point detection. First, change point algorithms typically require a minimum number of points within locally stationary segments, while in this scenario, all of our time series have change points in succession and several time series have an extreme outlier. The purpose of this scenario is to highlight measures that do not perform well in the case of extreme outliers, to identify measures that fail the triangle inequality and to investigate how badly they fail the triangle inequality.

The presence of outliers certainly impacts most of the distance measures when analysing the magnitude of the eigenvalues (seen in Figure \ref{fig:Worst_Eigenvalue_Analysis}). The Hausdorff distance in particular (Figure \ref{fig:WorstHausdorffEigenvalues}) fails to identify the similarity in the first five time series (1-5). Interestingly, the MH$_1$ does not display the appropriate degree of similarity in the first five time series, while the MH$_2$ and MH$_3$ measures produce plots that indicate similarity among the change point sets more appropriately. The Wasserstein distance in Figure \ref{fig:Worst_Wasserstein_Eigenvalues} and MJ$_1$ in Figure \ref{fig:Worst_MJ1_Eigenvalues} both produce outputs consistent with the simulation. Both the MJ$_{0.5}$ and MJ$_2$ also produce appropriate outputs.

\begin{figure}[t]
    \centering
    \begin{subfigure}[b]{0.24\textwidth}
        \includegraphics[width=\textwidth]{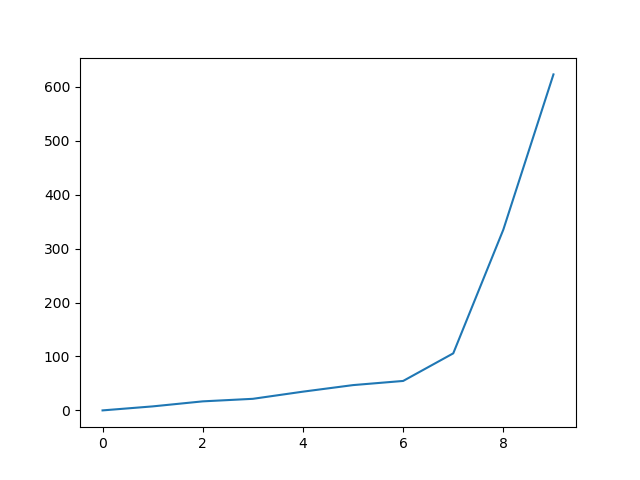}
        \caption{Hausdorff}
        \label{fig:WorstHausdorffEigenvalues}
    \end{subfigure}
        \begin{subfigure}[b]{0.24\textwidth}
        \includegraphics[width=\textwidth]{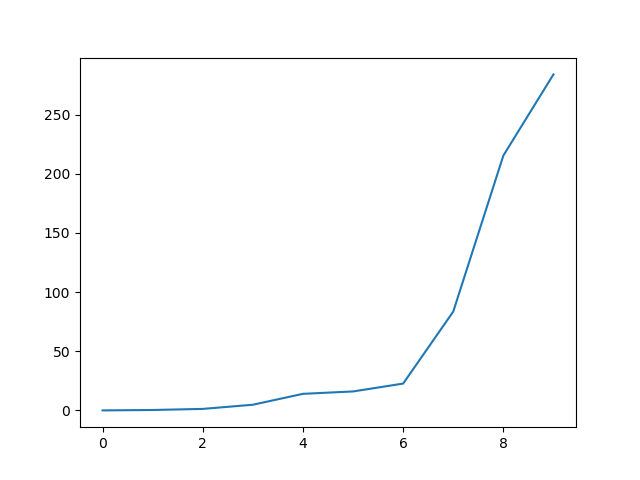}
        \caption{MH$_1$ }
        \label{fig:Worst_MH1_Eigenvalues}
    \end{subfigure}
        \begin{subfigure}[b]{0.24\textwidth}
        \includegraphics[width=\textwidth]{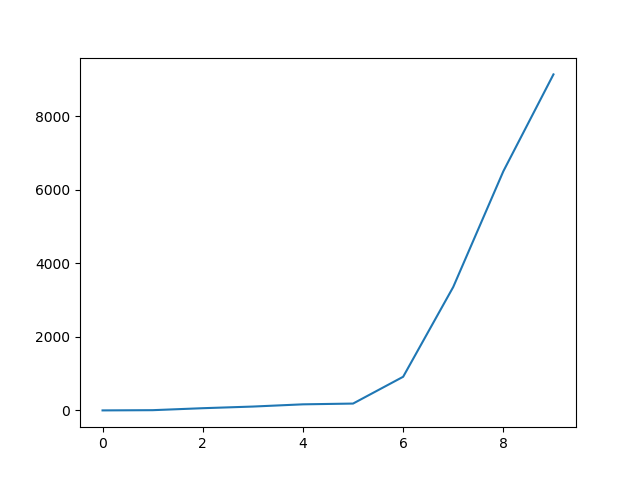}
        \caption{MH$_2$}
        \label{fig:Worst_MH2_Eigenvalues}
    \end{subfigure}
        \begin{subfigure}[b]{0.24\textwidth}
        \includegraphics[width=\textwidth]{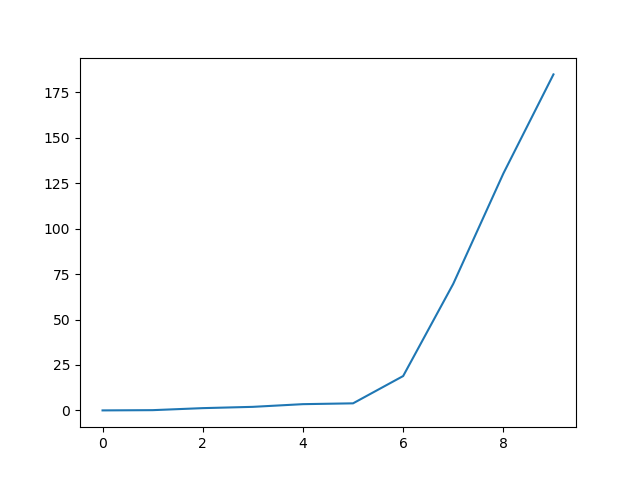}
        \caption{MH$_3$}
        \label{fig:Worst_MH3_Eigenvalues}
    \end{subfigure}
    \begin{subfigure}[b]{0.24\textwidth}
        \includegraphics[width=\textwidth]{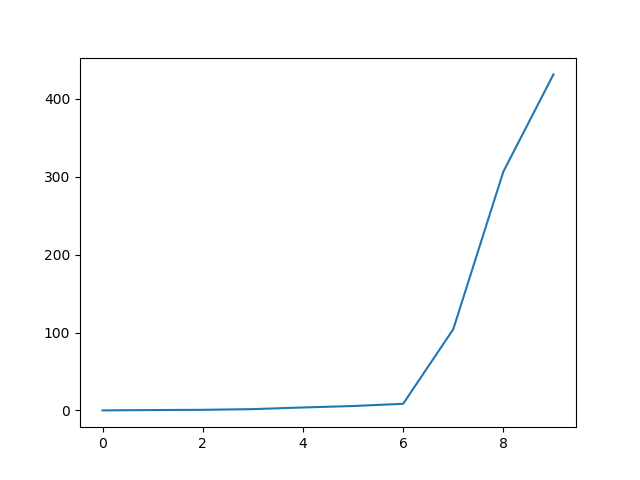}
        \caption{Wasserstein}
        \label{fig:Worst_Wasserstein_Eigenvalues}
    \end{subfigure}
    \begin{subfigure}[b]{0.24\textwidth}
        \includegraphics[width=\textwidth]{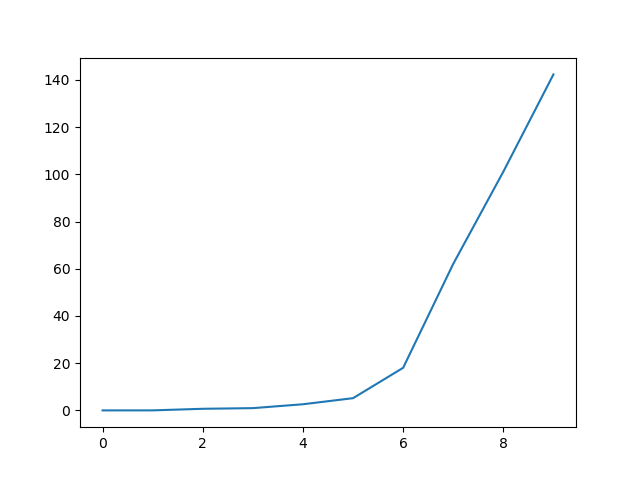}
        \caption{MJ$_{0.5}$}
        \label{fig:Worst_MJ05_Eigenvalues}
    \end{subfigure}
        \begin{subfigure}[b]{0.24\textwidth}
        \includegraphics[width=\textwidth]{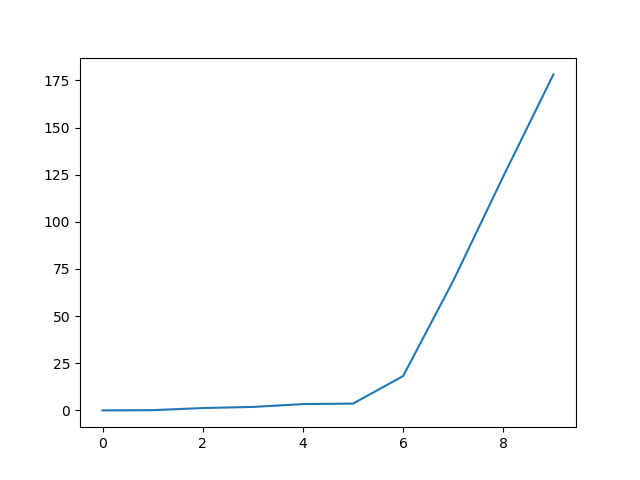}
        \caption{MJ$_1$}
        \label{fig:Worst_MJ1_Eigenvalues}
    \end{subfigure}
    \begin{subfigure}[b]{0.24\textwidth}
        \includegraphics[width=\textwidth]{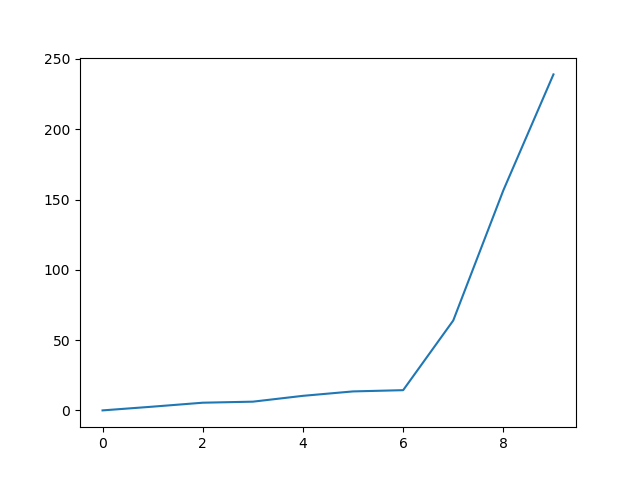}
        \caption{MJ$_2$}
        \label{fig:Worst_MJ2_Eigenvalues}
    \end{subfigure}
    \caption{Eigenvalue analysis with extreme outliers}\label{fig:Worst_Eigenvalue_Analysis}
\end{figure}

The dendrograms displayed in Figure \ref{fig:Worst_Dendrogram_Analysis} indicate that the Hausdorff distance performs worst. The heat map fails to identify appropriate similarity among the time series. The MH$_1$, MH$_2$, MH$_3$, MJ$_{0.5}$ and MJ$_1$ distances produce heat maps that are most representative of the true similarity in the set. Interestingly, the MJ$_2$ (Figure \ref{fig:Worst_MJ2_Dendrogram}) distance does not perform as well as the MJ$_1$ or MJ$_{0.5}$ (Figure \ref{fig:Worst_MJ1_Dendrogram}) in this scenario, perhaps due to the higher order of $p$. That is, the lower orders of $p$ provide stronger geometric mean averaging. In particular, the MJ$_2$ distance has particular difficulty distinguishing between clusters 2, 3 and 4, which should contain time series 6-8, 9 and 10 respectively. 

Spectral clustering highlights that the MH$_1$, MH$_2$, MH$_3$, MJ$_{0.5}$, MJ$_1$ and MJ$_2$ correctly identify clusters of similar time series. Again, both the Hausdorff and Wasserstein metrics do not perform well, as seen in Table \ref{tab:result_table_worst}. The Hausdorff metric in particular has severe sensitivity to outliers.

\begin{figure}[t]
    \centering
    \begin{subfigure}[b]{0.24\textwidth}
        \includegraphics[width=\textwidth]{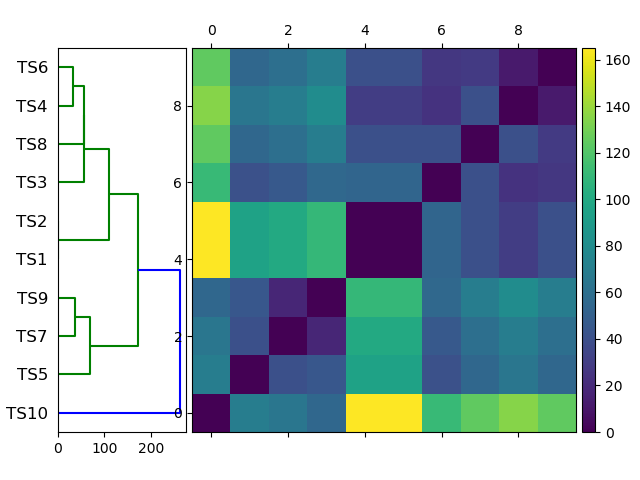}
        \caption{Hausdorff}
        \label{fig:WorstHausdorffDendrogram}
    \end{subfigure}
        \begin{subfigure}[b]{0.24\textwidth}
        \includegraphics[width=\textwidth]{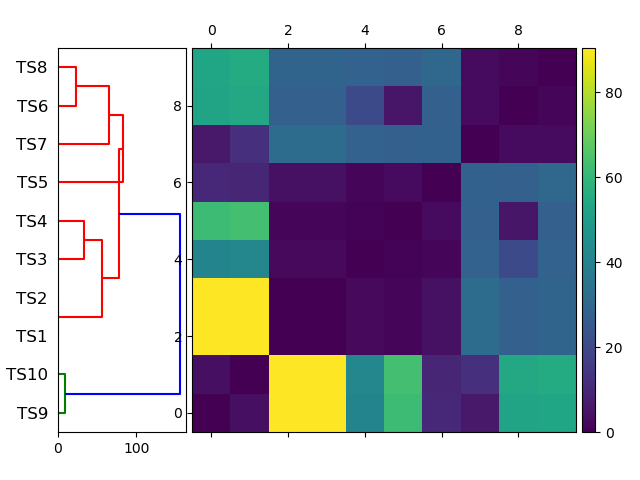}
        \caption{MH$_1$}
        \label{fig:Worst_MH1_Dendrogram}
    \end{subfigure}
        \begin{subfigure}[b]{0.24\textwidth}
        \includegraphics[width=\textwidth]{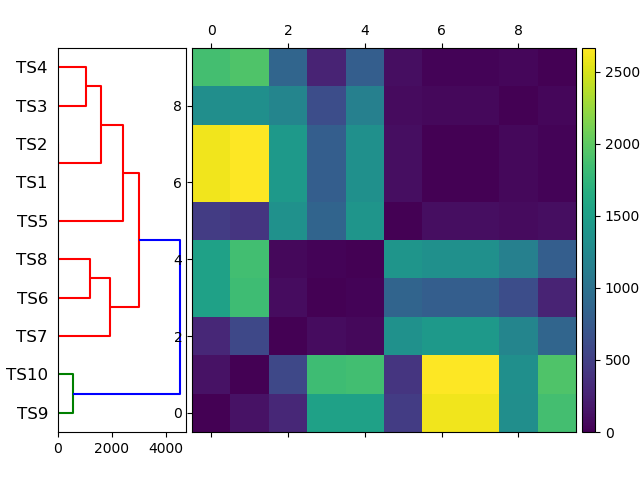}
        \caption{MH$_2$}
        \label{fig:Worst_MH2_Dendrogram}
    \end{subfigure}
        \begin{subfigure}[b]{0.24\textwidth}
        \includegraphics[width=\textwidth]{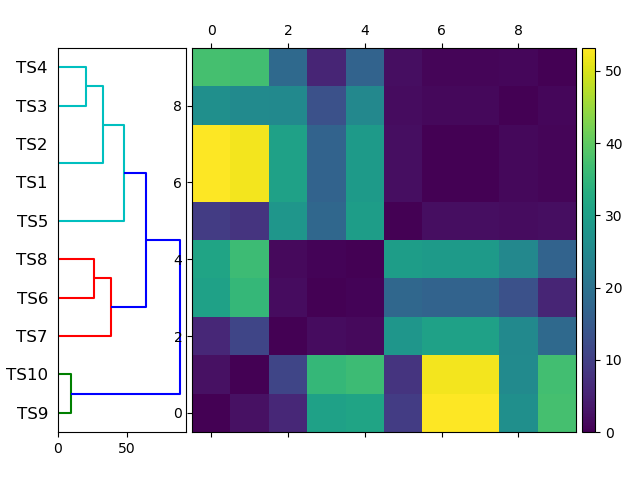}
        \caption{MH$_3$}
        \label{fig:Worst_MH3_Dendrogram}
    \end{subfigure}
    \begin{subfigure}[b]{0.24\textwidth}
        \includegraphics[width=\textwidth]{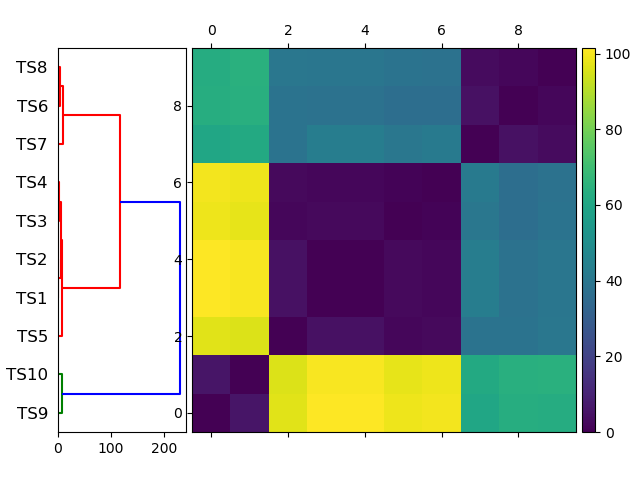}
        \caption{Wasserstein}
        \label{fig:Worst_Wasserstein_Dendrogram}
    \end{subfigure}
    \begin{subfigure}[b]{0.24\textwidth}
        \includegraphics[width=\textwidth]{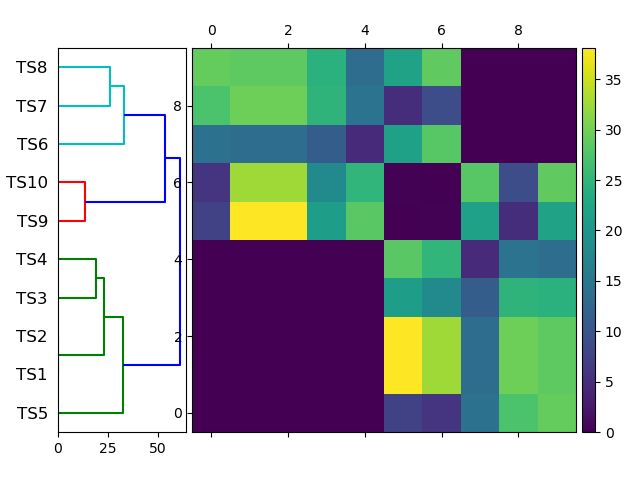}
        \caption{MJ$_{0.5}$}
        \label{fig:Worst_MJ05_Dendrogram}
    \end{subfigure}
        \begin{subfigure}[b]{0.24\textwidth}
        \includegraphics[width=\textwidth]{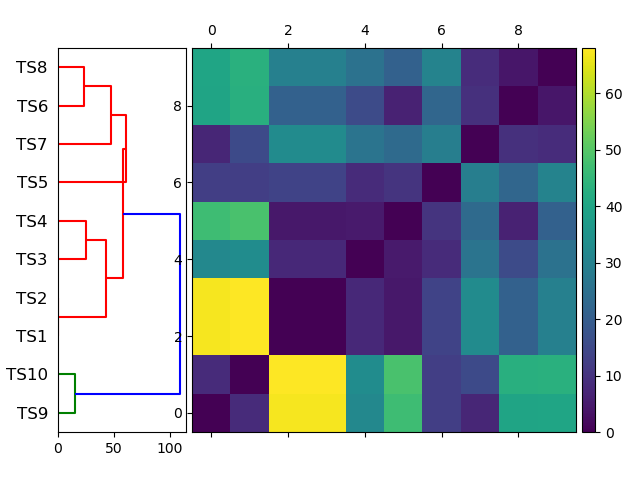}
        \caption{MJ$_1$}
        \label{fig:Worst_MJ1_Dendrogram}
    \end{subfigure}
    \begin{subfigure}[b]{0.24\textwidth}
        \includegraphics[width=\textwidth]{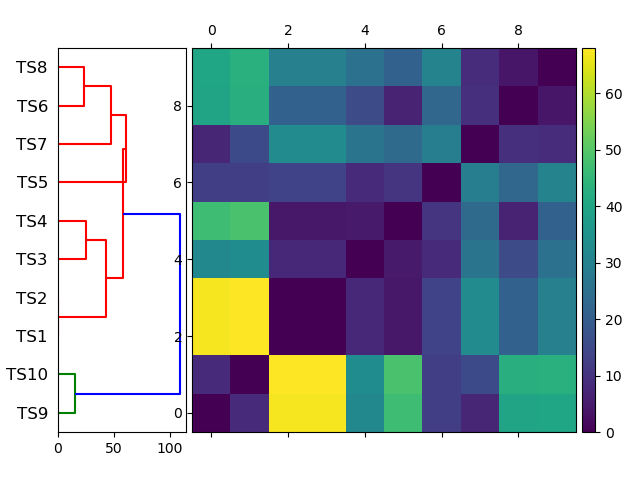}
        \caption{MJ$_2$}
        \label{fig:Worst_MJ2_Dendrogram}
    \end{subfigure}
    \caption{Dendrogram analysis extreme outliers}\label{fig:Worst_Dendrogram_Analysis}
\end{figure}

The transitivity analysis under this (contrived) scenario of extreme outliers demonstrates that the MJ$_{0.5}$ may be unusable, with 47.7\% of triples failing the triangle inequality and an average fail ratio of 8.63. The MJ$_1$ and MJ$_2$ distances have triples that fail the triangle inequality, but significantly less than MJ$_{0.5}$. In the case of the MJ$_1$ distance, 14.2\% of triples fail the triangle inequality, with an average fail ratio of 1.86. In the case of the MJ$_2$ distance, 11\% of triples fail the triangle inequality with an average fail ratio of 1.45. 14.4\% of MH$_1$ triples fail, with an average fail ratio of 2.18; 15\% of MH$_2$ triples fail, with an average fail ratio of 1.85; 15\% of MH$_3$ triples fail, with an average fail ratio of 1.86 (see Figure \ref{fig:Worst_Transitivity_Analysis}).

\begin{figure}[t]
    \centering
    \begin{subfigure}[b]{0.24\textwidth}
        \includegraphics[width=\textwidth]{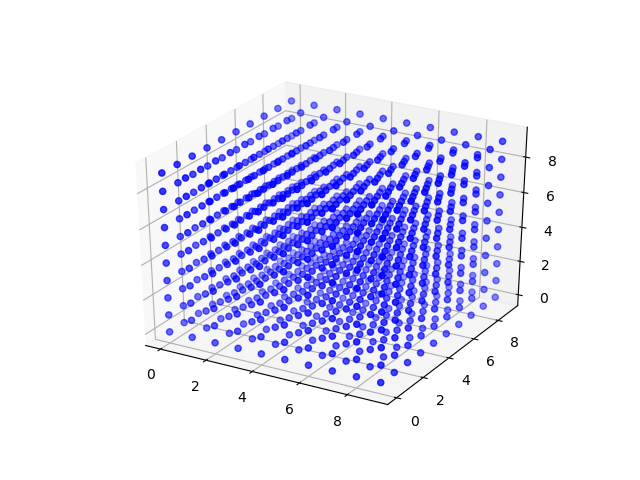}
        \caption{Hausdorff}
        \label{fig:WorstHausdorffTransitivity}
    \end{subfigure}
        \begin{subfigure}[b]{0.24\textwidth}
        \includegraphics[width=\textwidth]{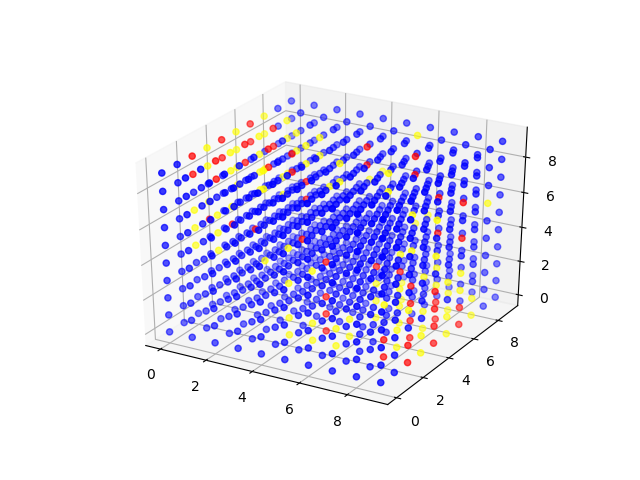}
        \caption{MH$_1$}
        \label{fig:Worst_MH1_Transitivity}
    \end{subfigure}
        \begin{subfigure}[b]{0.24\textwidth}
        \includegraphics[width=\textwidth]{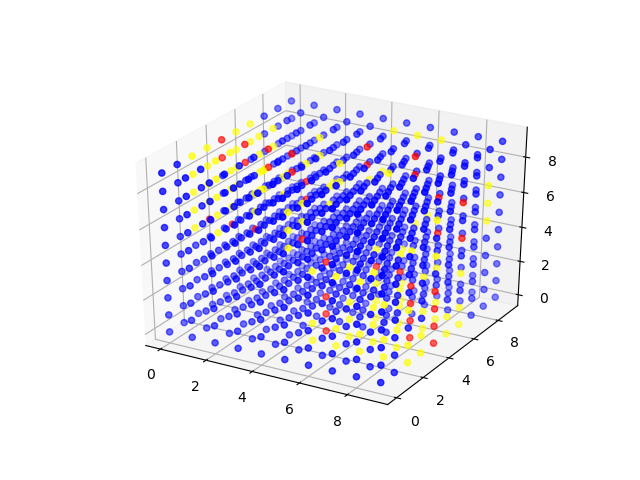}
        \caption{MH$_2$}
        \label{fig:Worst_MH2_Transitivity}
    \end{subfigure}
        \begin{subfigure}[b]{0.24\textwidth}
        \includegraphics[width=\textwidth]{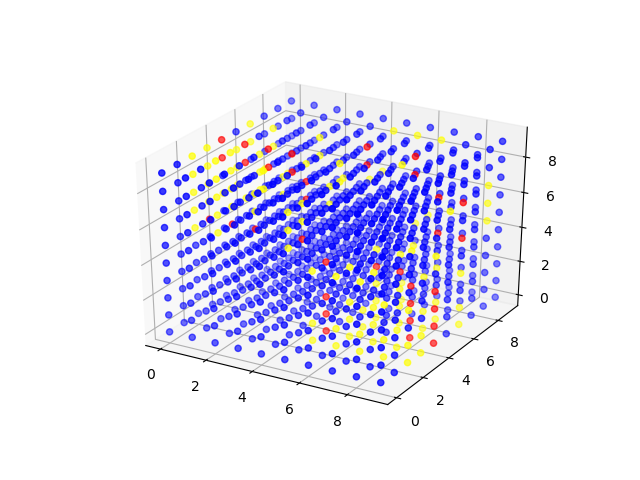}
        \caption{MH$_3$}
        \label{fig:Worst_MH3_Transitivity}
    \end{subfigure}
    \begin{subfigure}[b]{0.24\textwidth}
        \includegraphics[width=\textwidth]{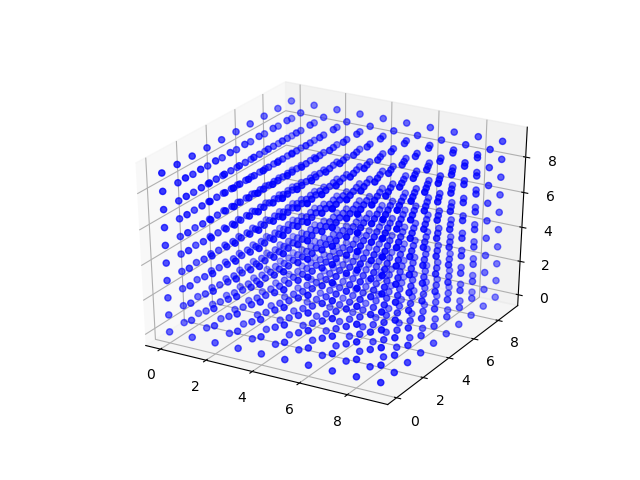}
        \caption{Wasserstein}
        \label{fig:Worst_Wasserstein_Transitivity}
    \end{subfigure}
    \begin{subfigure}[b]{0.24\textwidth}
        \includegraphics[width=\textwidth]{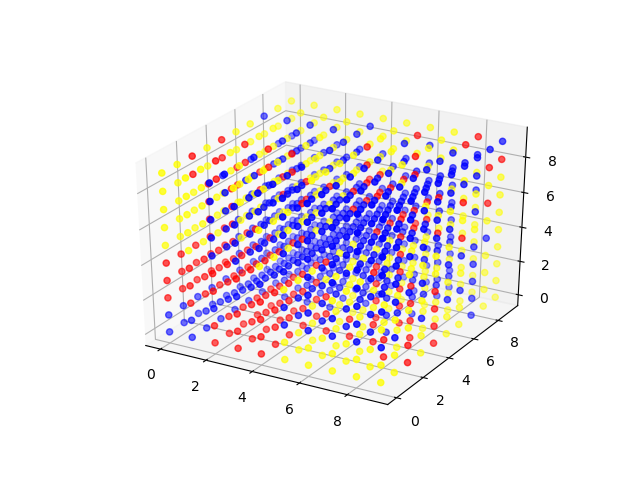}
        \caption{MJ$_{0.5}$}
        \label{fig:Worst_MJ05_Transitivity}
    \end{subfigure}
        \begin{subfigure}[b]{0.24\textwidth}
        \includegraphics[width=\textwidth]{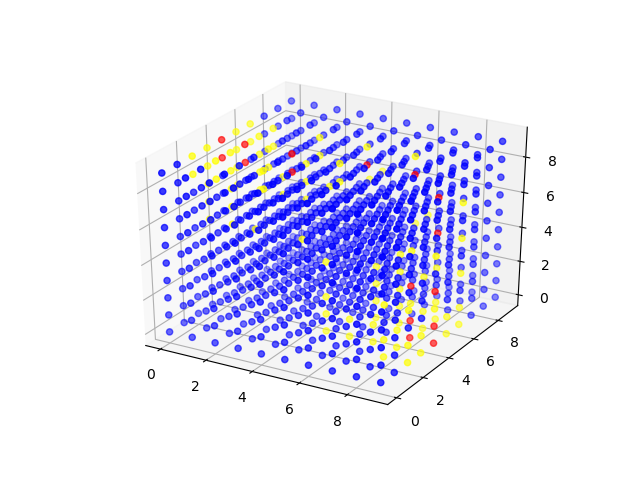}
        \caption{MJ$_1$}
        \label{fig:Worst_MJ1_Transitivity}
    \end{subfigure}
    \begin{subfigure}[b]{0.24\textwidth}
        \includegraphics[width=\textwidth]{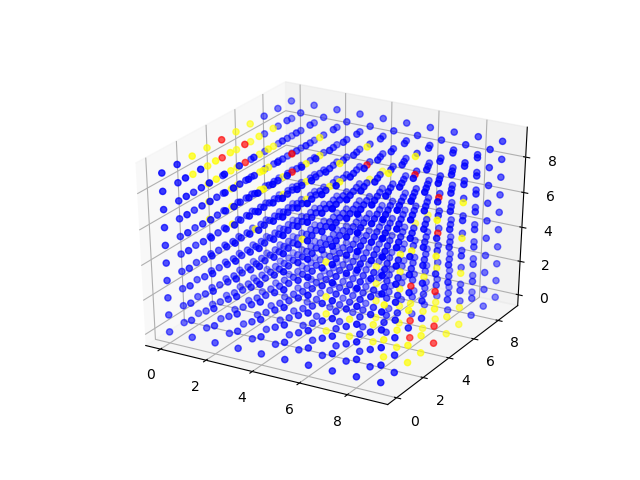}
        \caption{MJ$_2$}
        \label{fig:Worst_MJ2_Transitivity}
    \end{subfigure}
    \caption{Transitivity analysis extreme outliers}\label{fig:Worst_Transitivity_Analysis}
\end{figure}

\subsection{Summary of simulations}
The MJ$_1$, MJ$_2$ and modified Hausdorff semi-metrics all display an improvement over the traditional Hausdorff and Wasserstein metrics with regards to similarity and anomaly detection among collections of time series. The MJ$_p$ family does appear to provide an improvement over the modified Hausdorff semi-metrics in terms of inference. There are various cases where the MH$_1$, MH$_2$ or MH$_3$ exhibit errors in clustering experiments. Additionally the proportion of failed triples and severity of fails is not prohibitively worse in any of the scenarios explored. The order of $p$ in the MJ$_p$ semi-metrics has a significant impact on transitivity properties. The MJ$_{0.5}$ is clearly the worst violator of the triangle inequality. With this in mind, recalling Section \ref{new semi metrics section} and \cite{Dubuisson1994N} we again differ with Jain and Dubuisson's conclusions that MH$_1$ is the best for image matching. Both the MJ$_1$ and MJ$_2$ have advantages over MH$_1$. MJ$_{0.5}$ may be appropriate in situations where we do not care about the triangle inequality, or where we have verified through this analysis that it does not fail transitivity too severely. In such scenarios, MJ$_{0.5}$ is considerably more robust with respect to outliers than all other options.

% \subsection{Learn P to ensure transitivity holds*}

\subsection{Role of order $p$ in geometric averaging}
\label{choosing p section}
Given that there is a clear trade-off between various distance measures' transitivity and robust performance in the presence of outliers, one potential avenue to be explored would be optimising the order $p$ in the $L^p$ norm. That is, $p$ should be large enough to satisfy the triangle inequality, yet small enough to allow for geometric averaging. We find that when $p$ gets beyond 2 or 3, the geometric averaging property is lost and the measure becomes sensitive to outliers. Alternatively, one may allow an acceptable percentage of measurements within the time series collection to violate the triangle inequality. That is, we could insist that
\[
    \frac{|\{ (i,j,k):i,j,k \text{ distinct}, d(S_i,S_k)>d(S_i,S_j)+d(S_j,S_k)  \}|}{|\{ (i,j,k):i,j,k \text{ distinct}  \}|} < \epsilon
\]
%\begin{equation}
 %   R =
  %  \begin{cases}
   %   \text{1}, & \frac{D(i,k)}{D(i,j) + D(j,k)} \leq 1 \\
%      \text{0},  & \frac{D(i,k)}{D(i,j) + D(j,k)} > 1 \\
%    \end{cases}
%\end{equation}
%\begin{equation}
 %   \frac{\sum^{n}_{i=1} R_{i|i=1}}{\sum^{n}_{i=1} R_{i|i=1} + R_{i|i=0}} < \epsilon, 
%\end{equation}
where $\epsilon$ is some acceptable percentage of failed triples within the matrix. We set $\epsilon=0.05.$ Figure \ref{fig:PvsFails} demonstrates that once $p$ reaches $7$, less than 5\% of triples in $D$ fail the triangle inequality. However, the dendrogram in Figure \ref{fig:MJpDendrogram} demonstrates that the MJ$_7$ distance no longer provides the geometric averaging required to produce robust measurements. In fact, inference gained with the MJ$_7$ distance would be entirely erroneous. When considered in conjunction with the results from other simulated experiments, our findings suggest that $p$ needs to be low for powerful geometric averaging and robustness to outliers.

\begin{figure}[h]
    \centering
    \begin{subfigure}[b]{0.48\textwidth}
        \includegraphics[width=\textwidth]{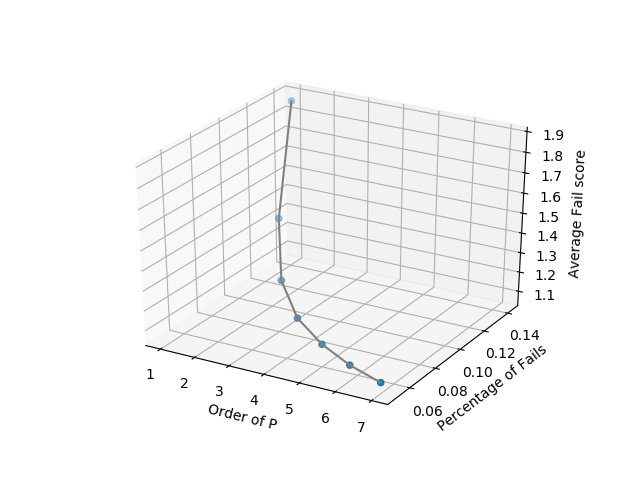}
        \caption{Order of $p$ vs percentage of failed triples and average failed triple ratio}
        \label{fig:PvsFails}
    \end{subfigure}
    \begin{subfigure}[b]{0.48\textwidth}
        \includegraphics[width=\textwidth]{"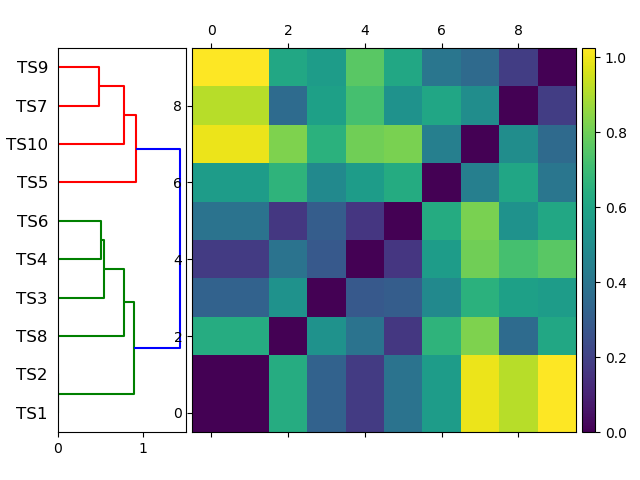}
        \caption{MJ$_7$}
        \label{fig:MJ$_7$ Dendrogram}
    \end{subfigure}
    \caption{MJ$_7$ dendrogram with extreme outliers}
    \label{fig:MJpDendrogram}
\end{figure}

\section{Real applications}
\label{real results}
\subsection{Cryptocurrency}
The cryptocurrency market is in its relative infancy in comparison to most exchange traded financial products. Cryptocurrencies are infamous for their volatile price behaviour, and high degree of correlation within the market, due to crowd behaviour often referred to as ``herding''. 

We analyse the similarity in the change points of the thirty largest cryptocurrencies by market capitalisation. We apply the Mann-Whitney test to the log returns of each cryptocurrency and use the MJ$_1$ distance as our measure of choice to allow powerful geometric averaging and robustness to outliers. Note that the log returns provide approximately distributed data with mean zero. Given the extreme volatility and associated kurtosis displayed by cryptocurrencies, this transformation is essential.

\begin{figure}[h]
    \centering
        \begin{subfigure}[b]{0.48\textwidth}
        \includegraphics[width=\textwidth]{"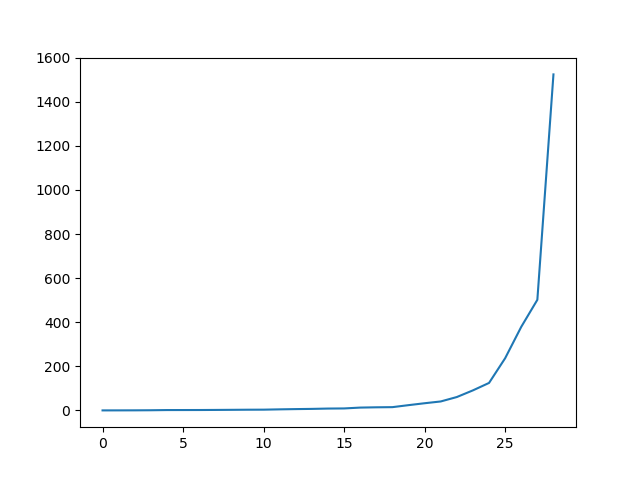}
        \caption{Cryptocurrency eigenvalues}
        \label{fig:Crypto Eigenvalues}
    \end{subfigure}    
    \begin{subfigure}[b]{0.48\textwidth}
        \includegraphics[width=\textwidth]{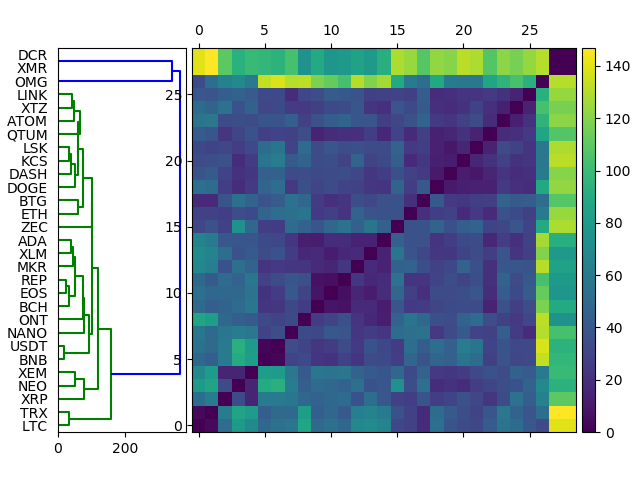}
        \caption{Cryptocurrency dendrogram}
        \label{fig:CryptoDendrogram}
    \end{subfigure}
    \begin{subfigure}[b]{0.48\textwidth}
        \includegraphics[width=\textwidth]{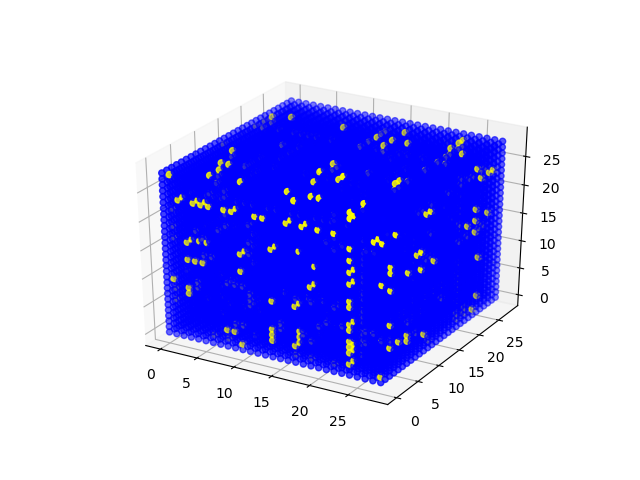}
        \caption{Cryptocurrency transitivity}
        \label{fig:CryptoTransitivity}
    \end{subfigure}
    \caption{Cryptocurrency change point analysis}\label{fig:CryptoAnalysis}
\end{figure}

Our method provides the following insights:
\begin{enumerate}
    \item The cryptocurrency market is characterised by a high degree of similarity between time series. This is seen in the plot of the eigenvalues (Figure \ref{fig:Crypto Eigenvalues}), which confirms that approximately 24 cryptocurrencies behave very similarly. The dendrogram in Figure \ref{fig:CryptoDendrogram} confirms this, with a large proportion of the market exhibiting a small distance between change points.
    \item Our method highlights the presence of anomalous cryptocurrencies, and indicates which cryptocurrencies are behaving dissimilarly to the rest of the market. In particular, Figure \ref{fig:CryptoDendrogram} shows that XMR and DCR behave very differently to the rest of the cryptocurrency market. They do however have a high degree of similarity between themselves.
    \item Spectral clustering on the distance matrix confirms the presence of anomalous cryptocurrencies warranting their own cluster.
    \item Finally, the dendrogram in Figure \ref{fig:CryptoDendrogram} highlights the presence subclusters of cryptocurrencies. These are subsets that behave similarly to one another, and less similarly to the rest of the market. This is often the case in financial markets, where companies in similar sectors or geographic regions may become correlated due to an exogenous variable or variables.
\end{enumerate}

To summarise, we have uncovered a high degree of correlation within the cryptocurrency market, and unearthed anomalies, including clusters thereof. Our analysis gives insight into anomalous and similar behaviours with respect to structural breaks, a feature of time series of key interest to analysts and market participants.

\subsection{19th century UK measles counts}
Given the recent public interest in COVID-19, we analyse historic counts of measles - a similarly infectious virus. In this context, understanding change points could have immense public health significance: signifying perhaps a growth in the infectivity, or a temporary scare where fears were allayed and the disease's spread halted. Applying our method to measles counts in 19th century UK cities, one can determine the similarity between structural changes in time series and perform anomaly detection simultaneously. First, we apply the Kolmogorov-Smirnov change point detection algorithm to all 7 time series, and yield 7 sets of change points (see Figure \ref{fig:MeaslesAnalysis}).

\begin{figure}[h]
    \centering
    \begin{subfigure}[b]{0.24\textwidth}
        \includegraphics[width=\textwidth]{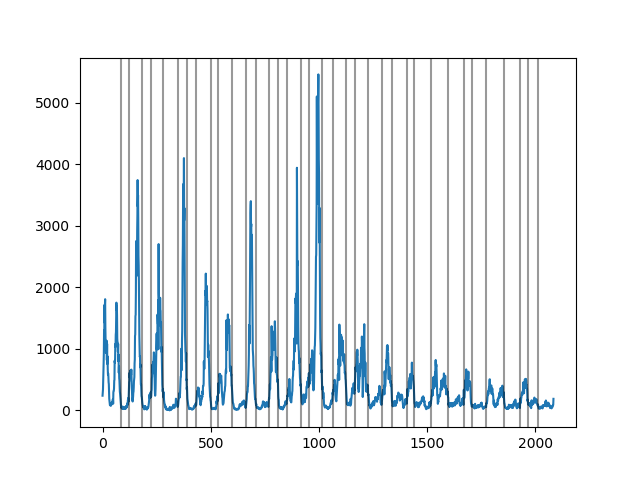}
        \caption{London}
        \label{fig:MeaslesLondon}
    \end{subfigure}
    \begin{subfigure}[b]{0.24\textwidth}
        \includegraphics[width=\textwidth]{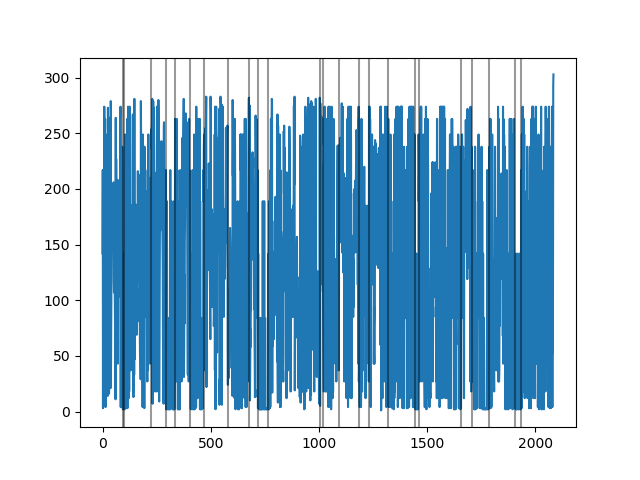}
        \caption{Bristol}
        \label{fig:MeaslesBristol}
    \end{subfigure}
    \begin{subfigure}[b]{0.24\textwidth}
        \includegraphics[width=\textwidth]{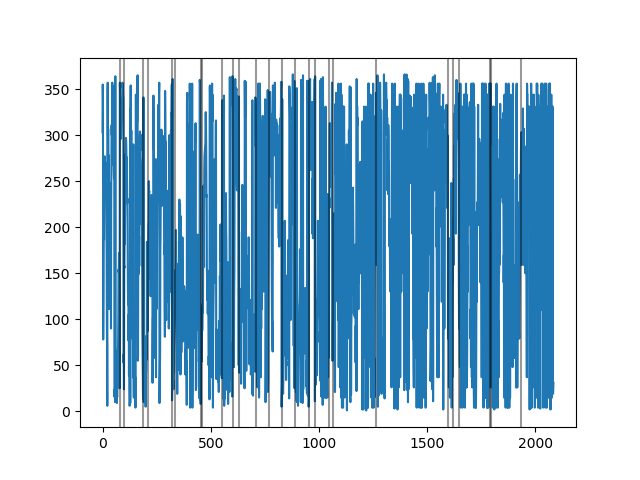}
        \caption{Liverpool}
        \label{fig:MeaslesLiverpool}
    \end{subfigure}
    \begin{subfigure}[b]{0.24\textwidth}
        \includegraphics[width=\textwidth]{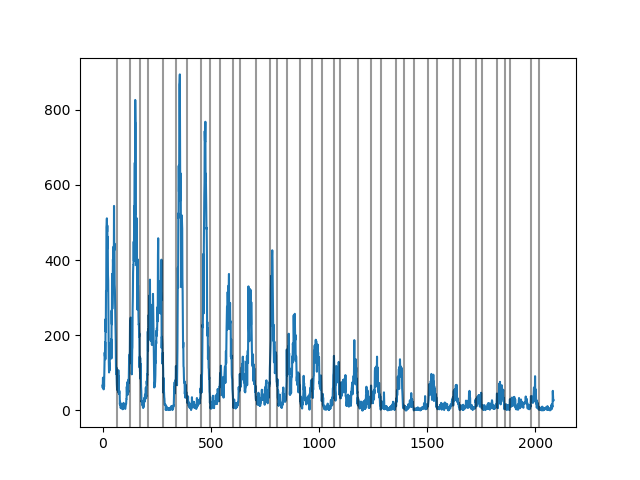}
        \caption{Manchester}
        \label{fig:MeaslesManchester}
    \end{subfigure}
    \begin{subfigure}[b]{0.24\textwidth}
        \includegraphics[width=\textwidth]{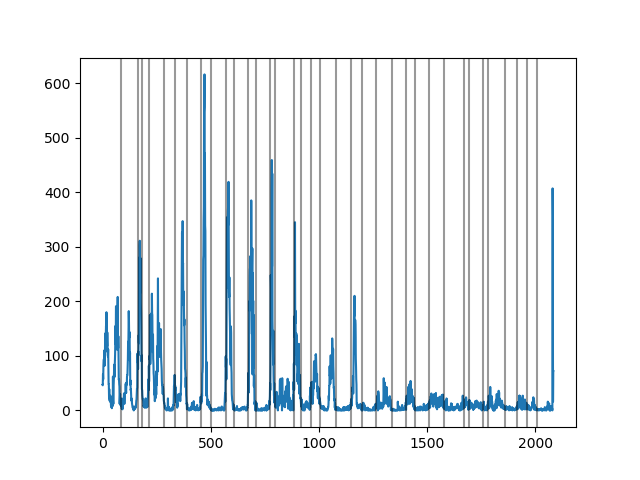}
        \caption{Newcastle}
        \label{fig:MeaslesNewcastle}
    \end{subfigure}
    \begin{subfigure}[b]{0.24\textwidth}
        \includegraphics[width=\textwidth]{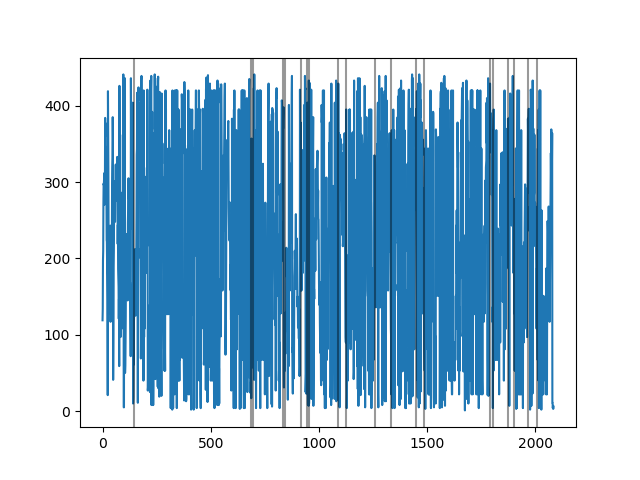}
        \caption{Birmingham}
        \label{fig:MeaslesBirmingham}
    \end{subfigure}
    \begin{subfigure}[b]{0.24\textwidth}
        \includegraphics[width=\textwidth]{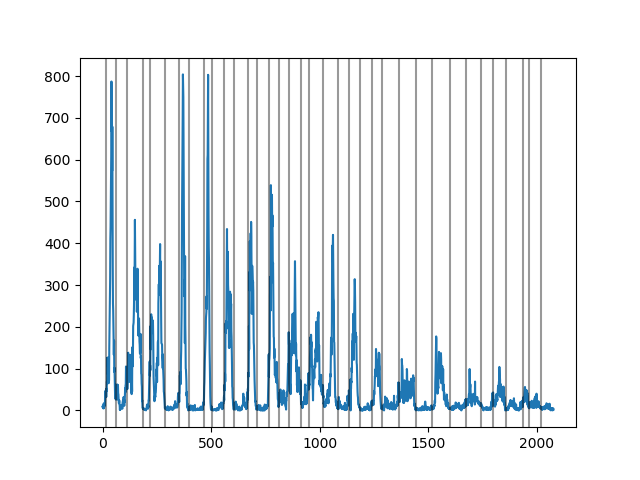}
        \caption{Sheffield}
        \label{fig:MeaslesSheffield}
    \end{subfigure}
    \caption{UK measles counts}
    \label{fig:MeaslesAnalysis}
\end{figure}

\begin{figure}[h]
    \centering
    \begin{subfigure}[b]{0.33\textwidth}
        \includegraphics[width=\textwidth]{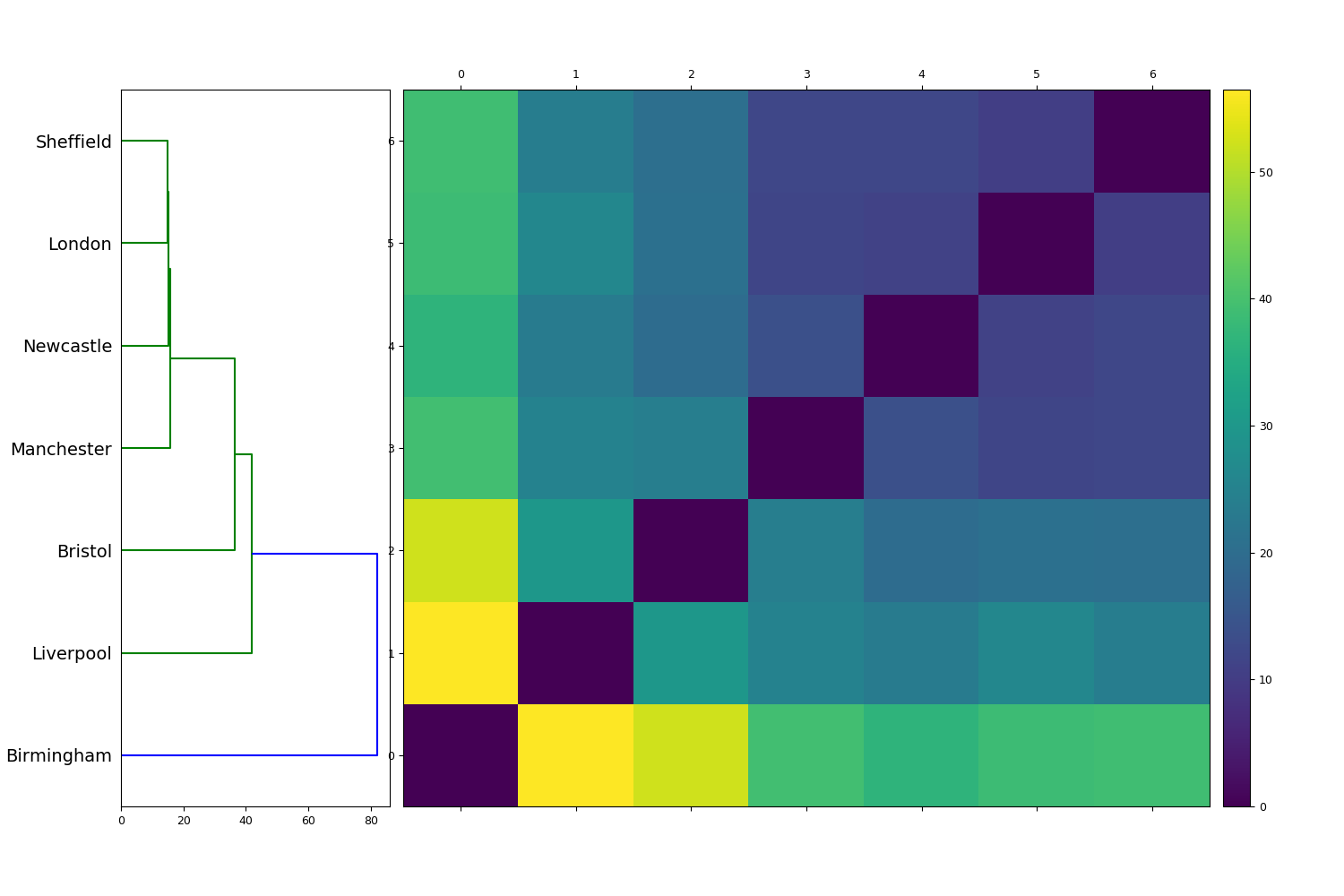}
        \caption{Measles dendrogram}
        \label{fig:MeaslesDendogram}
    \end{subfigure}
    \begin{subfigure}[b]{0.33\textwidth}
        \includegraphics[width=\textwidth]{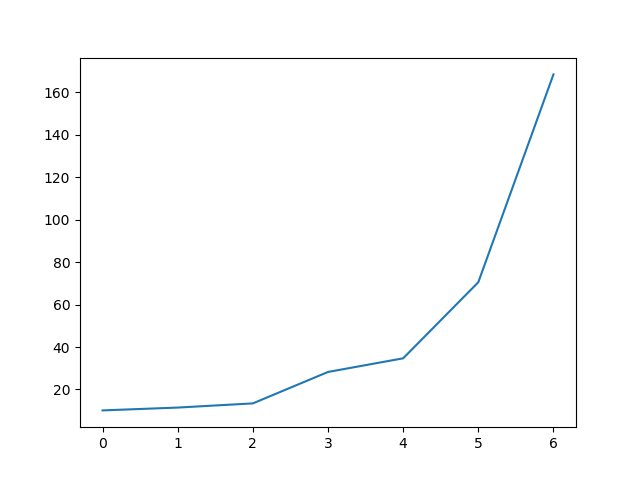}
        \caption{Measles eigenvalues}
        \label{fig:MeaslesEigenvalues}
    \end{subfigure}
    \begin{subfigure}[b]{0.33\textwidth}
        \includegraphics[width=\textwidth]{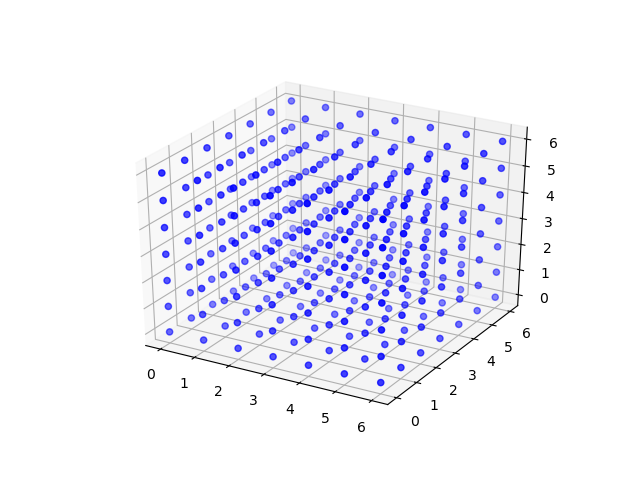}
        \caption{Measles transitivity}
        \label{fig:MeaslesTransitivity}
    \end{subfigure}
    \caption{UK measles change point analysis}\label{fig:MeaslesAnalysis2}
\end{figure}

Analysis on UK measles data yields the following insights:
\begin{enumerate}
    \item Figure \ref{fig:MeaslesEigenvalues} suggests that there are 5 time series that are similar and two relative outliers.
    \item Figure \ref{fig:MeaslesDendogram} highlights that there are 4 time series that are similar with respect to their structural breaks, 2 moderately similar and 1 anomalous city. The time series deemed to be most similar are Sheffield, London, Manchester and Newcastle. This insight is consistent with visual inspection of the time series in Figure \ref{fig:MeaslesAnalysis2}, where these time series exhibit multiple periodicities that are detected via the change point algorithm. Although our algorithm does not explicitly measure the periodic nature of each time series, our method manages to capture the most pertinent feature in this collection of time series.
    \item Birmingham is detected as the anomalous time series, and this is supported by the results from spectral clustering applied to the distance matrix.  
\end{enumerate}

\section{Conclusion}
\label{conclusion}
Prior work indicates that metrics adapting $L^p$ distance measures have mostly been used in computer vision applications. Although prior work has disproved the triangle inequality for these measurements, we are unaware of any work examining theoretically or empirically how badly the triangle inequality fails, and under what conditions. Our experiments on simulated and real data indicate that when there are more outliers within a collection of time series, and when there is a smaller average distance between successive change points - there is generally a higher percentage of time series that fail the triangle inequality. This is reflected in Proposition \ref{triangle inequality again} and its proof, showing that the bunching of points can cause the triangle inequality to fail up to any arbitrary constant. While the Hausdorff metric satisfies the triangle inequality, it is highly sensitive to outliers, as is confirmed in prior work among varying applications. This presents an interesting trade-off. Semi-metrics, such as the ones investigated in this paper, may lose universal transitivity, however different averaging methods, other than the $\max$ operation, produce more appropriate distance measurements between time series.

After applying a distance measure between sets of change points, we demonstrate the insights one can generate. %when analysing a distance between time series change points. 
First, we show a pithy means of eigenvalue analysis where we plot the absolute value of the ordered eigenvalues. This analysis illustrates the number of time series that are similar within any candidate scenario. We demonstrate the dendrogram's ability to simultaneously uncover similarity between time series and perform anomaly detection. Finally, we use spectral clustering on a transformation of the distance matrix (the affinity matrix) and illustrate how spectral clustering may determine groups of similar time series and detect clusters of anomalous time series.

Our results on simulated data show that semi-metrics perform as well or better than traditional metrics such as the Hausdorff and Wasserstein distances in multiple settings - no outliers, moderate outliers and extreme outliers. Moreover, our proposed family of semi-metrics perform better than the newer modified Hausdorff semi-metrics. %Perhaps the most significant benefit of the semi-metrics, is the averaging effect over outliers in the case of change points a long distance from the remainder of the set. 
 Our two applications, cryptocurrency and measles, demonstrate that our method may detect similar time series while also identifying anomalies.

In summary, %in determining an appropriate distance measure between sets of change points,
we have introduced a new computationally useful continuum of semi-metrics which we term \emph{MJ$_p$ distances} and applied them to measure distances between sets of change points. As $p \to \infty$, MJ$_p$ distances approach the Hausdorff distance. Thus, we have understood the Hausdorff distance in a new context within a family of MJ$_p$ distances.
Eigenvalue analysis, hierarchical clustering and spectral clustering can be used to analyse matrices of distances between sets of change points. Our analysis indicates that there is a trade-off between transitivity and sufficient geometric mean averaging (robust results in the presence of outliers). Traditional metrics such as the Hausdorff and Wasserstein distances are severely impacted by outliers, while semi-metrics with strong averaging properties such as the MJ$_{0.5}$ fail the triangle inequality frequently, and in some cases, severely. Our experiments indicate that both the MJ$_1$ and MJ$_2$ distances are better compromises than the existing modified Hausdorff framework. In future work, we will aim to investigate and generalise other distance metrics outside the Hausdorff framework.

%% The Appendices part is started with the command \appendix;https://www.overleaf.com/project/5e325b48187f42000144a0d9
\appendix
\section{\bf Change point detection algorithm}
\label{Appendix_CPD}
Many statistical modelling problems require the identification of change points in sequential data. The general setup for this problem is the following: a sequence of observations $x_1,x_2,...,x_n$ are drawn from random variables $X_1, X_2,...,X_n$ and undergo an unknown number of changes in distribution at points $\tau_1,...,\tau_m$. One assumes observations are independent and identically distributed between change points, that is, between each change points a random sampling of the distribution is occurring. Following Ross \cite{RossCPM}, we notate this as follows: \\

\begin{equation*}
    X_{i} \sim 
    \begin{cases}
      F_{0} \text{ if } i \leq \tau_1 \\
      F_{1} \text{ if } \tau_1 < i \leq  \tau_2  \\
      F_{2} \text{ if } \tau_2 < i  \leq \tau_3,  \\
      \hdots
    \end{cases}
\end{equation*}

While this requirement of independence may appear restrictive, dependence can generally be accounted for by modelling the underlying dynamics or drift process, then applying a change point algorithm to the model residuals or one-step-ahead prediction errors, as described by Gustafsson \cite{gustafsson2001}. The change point models applied in this paper follow Ross \cite{RossCPM}. %We briefly outline the two change point detection methods.

%The CPM extends techniques for batch detection to the sequential case. We first review the
%batch scenario, and then describe the sequential extension.
\subsection{Batch change detection (Phase I)}
This phase of change point detection is retrospective. We are given a fixed length sequence of observations $x_1,\ldots,x_n$ from random variables $X_1,\ldots,X_n$. For simplicity, assume at most one change point exists. If a change point exists at time $k$, observations have a distribution of $F_0$ prior to the change point, and a distribution of $F_1$ proceeding the change point, where $F_0 \neq F_1$. That is, one must test between the following two hypotheses for each $k$: 

\begin{equation*}
    H_0: X_{i} \sim F_0, i = 1,...,n
\end{equation*}

\begin{equation*}
    H_1: X_{i} \sim 
    \begin{cases}
      F_{0} & i = 1,2,...,k \\
      F_{1}, & i = k + 1, k+2, ..., n  \\
    \end{cases}
\end{equation*}
and end with the choice of the most suitable $k$.

One proceeds with a two-sample hypothesis test, where the choice of test is dependent on the assumptions about the underlying distributions. To avoid distributional assumptions, non-parametric tests can be used. Then one appropriately chooses a two-sample test statistic $D_{k,n}$ and a threshold $h_{k,n}$. If $D_{k,n}>h_{k,n}$ then the null hypothesis is rejected and we provisionally assume that a change point has occurred after $x_k$. These test statistics $D_{k,n}$ are normalised to have mean $0$ and variance $1$ and evaluated at all values $1 < k < n$, and the largest value is assumed to be coincident with the existence of our sole change point. That is, the test statistic is then

\begin{equation*}
    D_{n} = \max_{k=2,...,n-1} D_{k,n} = \max_{k=2,...,n-1} \Bigg| \frac{\Tilde{D}_{k,n} - \mu_{\Tilde{D}_{k,n}}}{\sigma_{\Tilde{D}_{k,n}}}  \Bigg|
\end{equation*}
where $\Tilde{D}_{k,n}$ were our unnormalised statistics.
\\

The null hypothesis of no change is then rejected if $D_{n} > h_n$ for some appropriately chosen
threshold $h_n$. In this circumstance, we conclude that a (unique) change point has occurred and its location is the value of $k$ which maximises $D_{k,n}$. That is,

\begin{equation*}
    \hat{\tau} = \argmax_k D_{k,n}.
\end{equation*}

This threshold $h_n$ is chosen to bound the Type 1 error rate as is standard in statistical hypothesis testing. First, one specifies an acceptable level $\alpha$ for the proportion of false positives, that is, the probability of falsely declaring that a change has occurred if in fact no
change has occurred. Then, $h_n$ should be chosen as the upper $\alpha$ quantile of the distribution
of $D_n$ under the null hypothesis. For the details of computation of this distribution, see \cite{RossCPM}. Computation can often be made easier by taking appropriate choice and storage of the $D_{k,n}$.

\subsection{Sequential change detection (Phase II)}
In this case, the sequence $(x_t)_{t \geq 1}$ does not have a fixed length. New observations are received over time, and multiple change points may be present. Assuming no change point exists so far, this approach treats $x_1, . . . , x_t$ as a fixed length sequence and computes $D_t$ as in phase I. A change is then flagged if $D_t > h_t$ for some appropriately chosen threshold. If no change is detected, the next observation $x_{t+1}$ is brought into the sequence. If a change is detected, the process restarts from the following observation in the sequence. The procedure therefore consists of a repeated sequence of hypothesis tests.

In this sequential setting, $h_t$ is chosen so that the probability of incurring a Type 1 error is constant over time, so that under the null hypothesis of no change, the following holds:

\begin{equation*}
    P(D_1 > h_1) = \alpha,
\end{equation*}

\begin{equation*}
    P(D_t > h_t | D_{t-1} \leq h_{t-1}, ... , D_{1} \leq h_{1}) = \alpha, \ t > 1.
\end{equation*}
In this case, assuming that no change occurs, the average number of observations received before a false positive detection occurs is equal to $\frac{1}{\alpha}$. This quantity is referred to as the average run length, or ARL0. Once again, there are computational difficulties with this conditional distribution and the appropriate values of $h_t$, as detailed in Ross \cite{RossCPM}.

\section{Proof of propositions}
\label{proof of props section}

\subsection{\bf Proposition \ref{deformation prop}} \label{A2}
\begin{proof}
Since $d^1_{MJ}$ is the average of two quantities while $d^{\text{MH}}_1$ is the maximum, we have $d^1_{MJ} \leq d^{\text{MH}}_1 \leq 2 d^1_{MJ}$ so these are equivalent as semi-metrics. 

Now let $S,T$ be two sets with the following properties: assume 
\begin{equation*}
    \frac{1}{|S|}\sum_{{s} \in {S}} d(s,T) < \frac{1}{|T|}\sum_{{t} \in {T}} d(t,S) 
\end{equation*}
Moreover, assume there exists an element $s_0 \in S$ that is the closest element of $S$ to every element of $T$. This property can hold easily if $S,T$ are each contained within convex sets $A,B$ respectively and $s_0$ is the closest element of $A$ to $B$.
Now slightly deform the elements of $S - \{s_0\}$, moving them a sufficiently small distance $\epsilon$ to produce another set $S'$ also containing $s_0$. Then

\begin{align*}
d^{\text{MH}}_1(S,T)=\max\Bigg(\frac{1}{|T|}\sum_{{t} \in {T}} d(t,S), \frac{1}{|S|}\sum_{{s} \in {S}} d(s,T)\Bigg)=\frac{1}{|T|}\sum_{{t} \in {T}} d(t,s_0)=d^{\text{MH}}_1(S',T) 
\end{align*}
However, if the elements of $S - \{s_0\}$ move, each distance $d(s,T)$ will vary continuously with their movement, so $d^1_{MJ}(S',T)$ will vary continuously with $S'$. This makes it a more precise measure of the distance between $S$ and $T$.
\end{proof}

\subsection{\bf Proposition \ref{duplication prop}} \label{A3}
\begin{proof}
Replace $S$ with a duplicated multiset $S_2$. \par \vspace{2mm} \noindent
Then $|S_2|=2|S|$, 
$\sum_{{t} \in {T}} d(t,S)=\sum_{{t} \in {T}} d(t,S_2)$, and $\sum_{{s} \in {S_2}} d(t,S)=2\sum_{s \in S} d(s,T)$.\par \vspace{2mm} \noindent
Thus, $d^1_{MJ}(S_2,T)=\bigg(\frac{1}{2|S|}(2\sum_{s \in S} d(s,T)) + \frac{1}{|T|}\sum_{t \in T} d(t,S) \bigg) = d^1_{MJ}(S,T)$, while
\begin{align*}
d^{\text{MH}}_2(S_2,T)=2\sum_{s \in S} d(s,T) + \sum_{t \in T} d(t,S) \neq d^{\text{MH}}_2(S_2,T)\\
d^{\text{MH}}_3(S_2,T)= \frac{1}{2|S|+|T|} \bigg(2\sum_{s \in S} d(s,T) + \sum_{t \in T} d(t,S) \bigg) \neq d^{\text{MH}}_3(S,T).
\end{align*}
\end{proof}

\subsection{\bf Proposition \ref{triangle inequality again}}  \label{A1}
%The non-existence of constant $k$ in the triangle inequality for semi-metrics $d^{\text{MH}}_i, \ i=1,2,3$.
\begin{proof}
Let $p>0$, and $i \in \{1,2,3\}$. All four distances MJ$_p$, MH$_1$, MH$_2$, MH$_3$ are symmetric in $S,T$ by definition, so axiom 2. holds.

All four distances are sums of powers of non-negative real numbers, so are clearly non-negative. Now assume either $d^p_{MJ}(S,T)$ or $d^{\text{MH}}_i(S,T)=0$ is zero for any $i=1,2,3$. Since all quantities in the summation of $d^{\text{MH}}_i$ or $d^p_{MJ}$ are non-negative, and all minimal distances $d(s,T),d(t,S)$ are included, this forces $d(s,T)=0$ for all $s \in S$, and $d(t,S)=0$ for all $t \in T$. That is, every element in $S$ lies in $T$ and vice versa. This proves $S=T$ and establishes axiom 1 for all four distances. So all four are semi-metrics.

Now turn to the triangle inequality. For MH$_2$, the triangle inequality fails easily when $S,R$ are much larger than $T$. For example, let $S=\{a_1,\ldots,a_n\}$ be a set of $n$ points very close to each other, let $R=\{b_1,\ldots,b_n\}$ be another such set, and $T=\{a,b \}$. Assume $a,a_1,\ldots,a_n$ and $b,b_1,\ldots,b_n$ respectively are all within $\epsilon$ of each other. Let $d(a,b)=d$ and let $\epsilon=o(d)$. Then 
\begin{eqnarray*}
d^{\text{MH}}_2(S,R) & = & 2nd(1+o(1)), \\
d^{\text{MH}}_2(S,T) & \leq & (n+1)\epsilon + (d+\epsilon) \hspace{2mm} \mbox{and} \hspace{2mm}  d^{\text{MH}}_2(T,R) \leq (n+1)\epsilon + (d+\epsilon).
\end{eqnarray*}
So $d^{\text{MH}}_2(S,T)+d^{\text{MH}}_2(T,R) \leq 2d(1+o(1))$ while $d^{\text{MH}}_2(S,R)=2nd(1+o(1))$. %This shows %$\frac{d^{\text{MH}}_2(S,R)}{d^{\text{MH}}_2(S,T)+d^{\text{MH}}_2(T,R)} \geq n - o(n) $ and 
If there were a universal $k$ such that a modified triangle inequality in (\ref{TriIn}) held, then setting $n>k$ would produce a contradiction. \\

MJ$_p$, MH$_1$ and MH$_3$ all contain averaging terms, so the triangle inequality can be violated by ``bunching'' of elements. For example, let $S=\{a,b\},R=\{b \}$, while $T=\{a,b,b_1,\ldots,b_{n-2}\}$. Assume as before $d(a,b)=d$ while $b,b_1,\ldots,b_n$ are all within $\epsilon$ of each other. Note we choose $n$ this time so that $T$ has $n$ elements, with $n-1$ bunched together. Then 

\begin{eqnarray*}
d^{\text{MH}}_1(S,T) & \leq & \frac{n-2}{n} \epsilon \leq \epsilon; d^{\text{MH}}_3(S,T)\leq \frac{(n-2)\epsilon}{n+2} \leq \epsilon\\
d^p_{MJ}(S,T) & \leq & \Big( \frac{1}{2n} (n-2)\epsilon^p \Big)^{\frac{1}{p}} \leq \epsilon \\
d^{\text{MH}}_1(T,R) & \leq & \frac{1}{n}d + \frac{n-2}{n}\epsilon \leq \frac{1}{n} d + \epsilon; d^{\text{MH}}_3(T,R) \leq \frac{1}{n+1}(d+(n-2)\epsilon) \leq \frac{1}{n} d + \epsilon, \\
d^p_{MJ}(T,R)& \leq & \left[ \frac{1}{2n}\Big( d^p+(n-2)\epsilon^p \Big) \right]^{\frac{1}{p}} \leq \Big(\frac{d^p}{2n} + \epsilon^p \Big)^{\frac{1}{p}}\\
%d^{\text{MH}}_1(T,R) & \leq & \frac{1}{n}d + \frac{n-1}{n}\epsilon \leq \frac{1}{n} d + \epsilon; d^{\text{MH}}_3(T,R) \leq \frac{1}{n+1}(d+(n-1)\epsilon) \leq \frac{1}{n} d + \epsilon,\\
d^{\text{MH}}_1(S,R) & = & \frac{d}{2}; d^{\text{MH}}_3(S,R)=\frac{d}{3}; d^p_{MJ}(S,R)=\Big( \frac{1}{4}d^p \Big)^{\frac{1}{p}} = 4^{\frac{-1}{p}}d
\end{eqnarray*}
%$$d^1(T,R) \leq \frac{1}{p}d + \frac{p-1}{p}\epsilon \leq \frac{1}{p} d + \epsilon; d^3(T,R) \leq \frac{1}{p+1}(d+(p-1)\epsilon) \leq \frac{1}{p} d + \epsilon$$
%$$d^1(S,R)=\frac{d}{2}, d^3(S,R)=\frac{d}{3}$$ 
So $d^{\text{MH}}_i(S,T)+d^{\text{MH}}_i(T,R)\leq \frac{1}{n} d+ 2\epsilon$; $ d^{\text{MH}}_i(S,R)\geq \frac{d}{3}$ for $i=1,3$. Choose $\epsilon < \frac{d}{2n}.$ Then $d^{\text{MH}}_i(S,T)+d^{\text{MH}}_i(T,R)\leq \frac{2d}{n} \leq \frac{6}{n}d^{\text{MH}}_i(S,R)$.\\

\noindent If there were a universal $k$ such that a modified triangle inequality in (\ref{TriIn}) held for $MH_i$, then setting $n>6k, \epsilon<\frac{d}{2n}$ would produce a contradiction.\\

%If $n$ is large and $\epsilon$ is small relative to $d$, we see there is no universal $k$ such that a modified triangle inequality holds.

\noindent For MJ$_p$, choose $\epsilon<d(2n)^{\frac{-1}{p}}<dn^{\frac{-1}{p}}$. Thus $d^p_{MJ}(S,T)+d^p_{MJ}(T,R) \leq \frac{2d}{n^{\frac{1}{p}}}$.%=O(d n^{\frac{-1}{p}})$ %\frac{d}{(2n)^{\frac{1}{p}}}$.

%And $d^p_{MJ}(S,T)+d^p_{MJ}(T,R) \leq \Big(\frac{d^p}{2n} + \epsilon^p \Big)^{\frac{1}{p}} + \epsilon$; $d^p_{MJ}(S,R) \geq 4^{\frac{-1}{p}}d$

\noindent Carefully noting what is above, $d^p_{MJ}(S,T)+d^p_{MJ}(T,R)= O(d n^{\frac{-1}{p}})$ while $d^p_{MJ}(S,R) =\Theta(d)$. Choosing $n$ sufficiently large, with $\epsilon<d(2n)^{\frac{-1}{p}}$, we deduce there is no universal modified triangle inequality for MJ$_p$.
\end{proof}

\subsection{\bf Proposition \ref{intersection prop}} \label{A4}
\begin{proof}
Let $d_H(S,T)=M.$ From definition \ref{New MJ distance}, 
\begin{equation*}
    d^p_{MJ}({S},{T}) = \Bigg(\frac{\sum_{t\in T} d(t,S)^p}{2|T|} + \frac{\sum_{{s} \in {S}} d(s,T)^p}{2|S|} \Bigg)^{\frac{1}{p}}.
\end{equation*}
Any $d(s,T)$ or $d(t,S)$ term with $s \in S\cap T $ or $t \in S\cap T$ respectively vanishes. Any other $d(s,T),d(t,S)$ term is at most $M.$ So
\begin{equation*}
    d^p_{MJ}({S},{T}) \leq \left[ \frac{1}{2|S|}(|S|-r) + \frac{1}{2|T|}(|T|-r)\right] M
\end{equation*}
which gives the inequality after simplifying.

Now consider the Wasserstein distance. Let $A=\{0,1,...,n-1\}$ and $B=\{1,2,...,n\}$. Then $|A|=|B|=n$ while $|A\cap B|=n-1.$ Clearly $d_H(A,B)=1;$ $d_W(A,B)$ may be computed \cite{DelBarrio} as 
\begin{align*}
    \int_{\mathbb{R}} |F - G| dx
\end{align*}
where $F,G$ are the cumulative density functions associated to $\mu_A$ and $\mu_B$ as in \ref{Wasserstein delta}. Both $F,G$ are piecewise constant increasing step functions. By integrating $\mu_A,\mu_B$,

$$F=\sum_{j=1}^{n-1} \frac{j}{n} \mathbbm{1}_{[j-1,j)} + \mathbbm{1}_{[n-1,\infty)}, G= \sum_{j=1}^{n-1} \frac{j}{n}\mathbbm{1}_{[j,j+1)} + \mathbbm{1}_{[n,\infty)} $$

Hence $F-G=\frac{1}{n} \mathbbm{1}_{[0,n)}$ and $ d_W(A,B)=\int_{\mathbb{R}} |F - G| dx =1.$

%\begin{equation*}
 %   F(x) =  
  %  \begin{cases}
   %   \frac{j}{n} & j=1,...,n \\
    %  F_{1}, & i = k + 1, k+2, ..., n  \\
    %\end{cases}
%\end{equation*}

%$F$ takes value $\frac{j}{n}$ over  and differ by $\frac{1}{n}$ over $[0,n]$. Hence the above integral computes $d_W(A,B)=1$, coinciding with $d_H(A,B)$. 
By contrast, $d^p_{MJ}(A,B)=(\frac{1}{n})^\frac{1}{p}$, which is the bound given by the inequality.
\end{proof}


%% References



%% New version of the num-names style
\bibliographystyle{elsarticle-num-names}
\bibliography{References.bib}
\end{nolinenumbers}





% \end{thebibliography}

\end{document}